\documentclass{article}

\PassOptionsToPackage{numbers, compress}{natbib}

\usepackage[final]{neurips_2022}

\title{Near-Optimal Sample Complexity Bounds \\ for Constrained MDPs}

\author{%
  Sharan Vaswani\thanks{Equal contribution} \\
  Simon Fraser University\\
  \texttt{vaswani.sharan@gmail.com} \\
  \And
  Lin F. Yang$^*$ \\
  University of California, Los Angeles \\
  \texttt{linyang@ee.ucla.edu} \\
  \And
  Csaba Szepesv\'ari \\
  Amii, University of Alberta, DeepMind \\
  \texttt{szepesva@ualberta.ca} \\
}

\usepackage[utf8]{inputenc} 
\usepackage[T1]{fontenc}    
\usepackage{url}            
\usepackage{booktabs}       
\usepackage{amsfonts}       
\usepackage{xcolor}         
\usepackage{amsthm}

\usepackage{amsmath}
\usepackage{mathrsfs}
\usepackage{amssymb}
\usepackage{mathtools}
\usepackage{bm}
\usepackage{datetime}
\usepackage{tikz}
\usepackage{listings}
\usepackage{mdframed}
\usepackage{hyperref}
\hypersetup{
    unicode=false,          
    pdftoolbar=true,        
    pdfmenubar=true,        
    pdffitwindow=false,     
    pdfstartview={FitH},    
    pdftitle={XXXXX},    
    pdfauthor={XXX},     
    pdfsubject={Bandits, Reinforcement Learning},   
    pdfcreator={Creator},   
    pdfproducer={Producer}, 
    pdfkeywords={bandits} {reinforcement learning} {policy gradient}, 
    pdfnewwindow=true,      
    colorlinks=true,       
    linkcolor=blue,          
    citecolor=blue,        
    filecolor=magenta,      
    urlcolor=cyan           
}
\usepackage{times}
\usepackage{nicefrac}
\usepackage{wrapfig}

\usepackage{xpatch}
\usepackage{thm-restate}
\usepackage{subfigure}
\definecolor{shadecolor}{gray}{0.90}
\declaretheoremstyle[
headfont=\normalfont\bfseries,
notefont=\mdseries, notebraces={(}{)},
bodyfont=\normalfont,
postheadspace=0.5em,
spaceabove=5pt,
mdframed={
  skipabove=3pt,
  skipbelow=3pt,
  hidealllines=true,
  backgroundcolor={shadecolor},
  innerleftmargin=2pt,
  innerrightmargin=2pt}
]{shaded}

\usepackage[algoruled, linesnumbered]{algorithm2e}
\usepackage[capitalize]{cleveref}

\declaretheorem[style=shaded]{theorem}

\newtheorem{lemma}[theorem]{Lemma}

\newtheorem{definition}[theorem]{Definition}

\newenvironment{proofsketch}[1][]{\begin{proof}[Proof Sketch:]}{\end{proof}}

\crefname{alg}{Algorithm}{Algorithm}

\usepackage{dsfont}
\usepackage{colortbl}
\usepackage{pifont}
\usepackage{microtype} 
\usepackage{float}
\usepackage{xfrac} 
\usepackage{xspace}

\makeatletter
\renewcommand*{\@textcolor}[3]{%
  \protect\leavevmode
  \begingroup
    \color#1{#2}#3%
  \endgroup
}
\makeatother

\mdfdefinestyle{mdframedthmbox}{%
	leftmargin=.0\textwidth,
	rightmargin=.0\textwidth,%
	innertopmargin=0.75em,
	innerleftmargin=.5em,
	innerrightmargin=.5em,
}
\newenvironment{thmbox}
	{%
		\begin{mdframed}[style=mdframedthmbox]%
	}{%
		\end{mdframed}%
	}

\usepackage{wrapfig}

\newcommand{\norminf}[1]{\left\|#1\right\|_{\infty}}
\newcommand{\abs}[1]{\left\vert #1 \right\vert}

\newcommand{\R}{\mathbb{R}}

\newcommand{\cA}{\mathcal{A}}

\newcommand{\cE}{\mathcal{E}}

\newcommand{\cK}{\mathcal{K}}

\newcommand{\cS}{\mathcal{S}}

\newcommand{\epsp}{\varepsilon_{\text{\tiny{opt}}}}
\newcommand{\epsl}{\varepsilon_{\text{\tiny{l}}}}

\renewcommand{\epsilon}{\varepsilon}

\DeclareMathOperator*{\argmin}{arg\ min}
\DeclareMathOperator*{\argmax}{arg\ max}

\newcommand{\Mhat}{\hat{M}}

\newcommand{\Var}[2]{\text{Var}_{#1}(#2)}

\newcommand{\const}[1]{V_c^{#1}(\rho)}

\newcommand{\consthat}[1]{\hat{V}_c^{#1}(\rho)}

\newcommand{\reward}[1]{V_r^{#1}(\rho)}
\newcommand{\rewardp}[1]{V_{r_p}^{#1}(\rho)}

\newcommand{\rewardrp}[1]{V_{r_p}^{#1}}

\newcommand{\rewardrphat}[1]{\hat{V}_{r_p}^{#1}}

\newcommand{\rewarda}[1]{V_{\alpha}^{#1}}
\newcommand{\rewardb}[1]{V_{\beta}^{#1}}

\newcommand{\rewardhat}[1]{\hat{V}_r^{#1}(\rho)}
\newcommand{\rewardhatp}[1]{\hat{V}_{r_p}^{#1}(\rho)}

\newcommand{\rewardahat}[1]{\hat{V}_{\alpha}^{#1}}
\newcommand{\rewardbhat}[1]{\hat{V}_{\beta}^{#1}}

\newcommand{\optreward}{V_r^{*}(\rho)}

\newcommand{\lag}[2]{V_{l}^{#1,#2}(\rho)}

\newcommand{\rewardq}[1]{Q_{r}^{#1}}

\newcommand{\constq}[1]{Q_{c}^{#1}}

\newcommand{\piopt}{\pi^*}
\newcommand{\pihat}{\hat{\pi}}
\newcommand{\pihatopt}{\hat{\pi}^*}
\newcommand{\pitildeopt}{\tilde{\pi}^*}
\newcommand{\pihatopta}{\hat{\pi}_\alpha^{*}}

\newcommand{\pit}{\pi_t}
\newcommand{\pihatt}{\hat{\pi}_t}
\newcommand{\lambdat}{\lambda_t}
\newcommand{\lambdatt}{\lambda_{t+1}}
\newcommand{\PP}{\mathbb{P}}

\newcommand{\pihatbar}{\bar{\pi}_{T}}

\newcommand{\cP}{\mathcal{P}}
\newcommand{\cPhat}{\hat{\mathcal{P}}}

\newcommand{\tautil}{\tilde{\tau}}
\newcommand{\pitil}{\tilde{\pi}}

\begin{document}
\maketitle

\begin{abstract}
In contrast to the advances in characterizing the sample complexity for solving Markov decision processes (MDPs), the optimal statistical complexity for solving constrained MDPs (CMDPs) remains unknown. We resolve this question by providing \emph{minimax} upper and lower bounds on the sample complexity for learning near-optimal policies in a discounted CMDP with access to a generative model (simulator). In particular, we design a model-based algorithm that addresses two settings: (i) \emph{relaxed feasibility},  where small constraint violations are allowed, and (ii) \emph{strict feasibility}, where the output policy is required to satisfy the constraint. For (i), we prove that our algorithm returns an $\epsilon$-optimal policy with probability $1 - \delta$, by making  $\tilde{O}\left(\frac{S A \log(1/\delta)}{(1 - \gamma)^3 \epsilon^2}\right)$ queries to the generative model, thus matching the sample-complexity for unconstrained MDPs. For (ii), we show that the algorithm's sample complexity is upper-bounded by $\tilde{O} \left(\frac{S A \, \log(1/\delta)}{(1 - \gamma)^5 \, \epsilon^2 \zeta^2} \right)$ where $\zeta$ is the problem-dependent Slater constant that characterizes the size of the feasible region. Finally, we prove a matching lower-bound for the strict feasibility setting, thus obtaining the first near minimax optimal bounds for discounted CMDPs. Our results show that learning CMDPs is as easy as MDPs when small constraint violations are allowed, but inherently more difficult when we demand zero constraint violation. 
\end{abstract}
\section{Introduction}
\label{sec:introduction}
Common reinforcement learning (RL) algorithms focus on optimizing an unconstrained objective, and have found applications in games such as Atari~\citep{mnih2015human} or Go~\citep{silver2016mastering}, robot manipulation tasks~\citep{tan2018sim,zeng2020tossingbot} or clinical trials~\citep{schaefer2005modeling}. However, many applications require the planning agent to satisfy constraints -- for example, in wireless sensor networks~\citep{buratti2009overview} where there is a constraint on  average power consumption. More generally, in the constrained Markov decision processes (CMDP) framework, the goal is to find a policy that maximizes the value associated with a reward function subject to the policy achieving a return (for a second reward function) that exceeds an apriori determined  threshold \citep{altman1999constrained}. There has been substantial work addressing the planning problem to find a near-optimal policy in a known CMDP~\citep{borkar2005actor,borkar2014risk, tessler2018reward,paternain2019constrained,achiam2017constrained, xu2021crpo}. However, since the CMDP is unknown in most practical applications, we consider the problem of finding a near-optimal policy in this more challenging setting. 

There have been multiple recent approaches to obtain a near-optimal policy in CMDPs in the regret-minimization or PAC-RL settings~\citep{efroni2020exploration, zheng2020constrained,brantley2020constrained,kalagarla2020sample,wachi2020safe, miryoosefi2022simple, yu2021provably, ding2021provably, gattami2021reinforcement,hasan2021model, chen2021primal}. These works tackle the exploration, estimation and planning problems simultaneously. On the other hand, recent works~\citep{hasan2021model, wei2021provably,bai2021achieving} consider an easier, but even more fundamental problem of obtaining a near-optimal policy with access to a simulator or \emph{generative model}~\citep{kearns1999finite,kakade2003sample,agarwal2020model}. In particular, these works assume that the transition probabilities in the underlying CMDP are unknown, but the planner has access to a sampling oracle (the generative model) that returns a sample of the next state when given any state-action pair as input. This is the problem setting we consider and \emph{aim to obtain matching upper and lower bounds on the sample complexity of planning in CMDPs with access to a generative model}.

Given a target error $\epsilon > 0$, the approximate CMDP objective is to return a policy that achieves a cumulative reward within an $\epsilon$ additive error of the optimal policy in the CMDP. Previous work can be classified into two categories based on how it tackles the constraint -- for the easier problem that we term \emph{relaxed feasibility}, the policy returned by an algorithm is allowed to violate the constraint by at most $\epsilon$. On the other hand, for the more difficult \emph{strict feasibility} problem, the returned policy is required to strictly satisfy the constraint and achieve zero constraint violation. Except for the recent works of~\citet{wei2021provably} and \citet{bai2021achieving}, most provably efficient approaches including those in the regret-minimization and PAC-RL settings consider the relaxed feasibility setting. For this problem, the best model-based algorithm requires $\tilde{O}\left(\frac{S^2 A}{(1 - \gamma)^3 \epsilon^2}\right)$ samples to return an $\epsilon$-optimal policy in an infinite-horizon $\gamma$-discounted CMDP with $S$ states and $A$ actions
\citep{hasan2021model}, while the best model-free approach requires $\tilde{O}\left(\frac{S A}{(1 - \gamma)^5 \epsilon^2}\right)$ samples for achieving the objective~\citep{ding2021provably}. On the other hand, the best known upper bounds for a model-free algorithm in the strict feasibility setting are achieved by~\citet{bai2021achieving}. In particular, their algorithm requires $\tilde{O}\left(\frac{S A}{(1 - \gamma)^2 \epsilon^2}\right)$ samples~\citep[Theorem 2]{bai2021achieving} to output an $\epsilon$-optimal policy. However, their analysis considers normalized reward and constraint value functions~\citep[Eq. 1]{bai2021achieving} that lie in the $[0,1]$ range (compared to the standard $[0, 1/1- \gamma]$ range). This difference in the scale of the values prevents a direct comparison of their results to our sample complexity bounds. Subsequently, we show that when appropriately normalized, our sample complexity bounds are better by a $\left(\nicefrac{1}{1 - \gamma}\right)$ factor in both the relaxed feasibility (\cref{sec:ub-relaxed})  and strict feasibility settings (\cref{sec:ub-strict}). 

Importantly, there are no lower bounds characterizing the difficulty of either the relaxed or strict feasibility problems (except in degenerate cases where the constraint is always satisfied and the CMDP problem reduces to an unconstrained MDP). To get an indication of what the optimal bounds might be, it is instructive to compare these results to the unconstrained MDP setting. For unconstrained MDPs with access to a generative model, both model-based~\citep{agarwal2020model,li2020breaking} and model-free approaches~\citep{sidford2018near} can return an $\epsilon$-optimal policy within near-optimal $\tilde{\Theta} \left(\frac{S A}{(1 - \gamma)^3 \epsilon^2}\right)$ sample-complexity~\citep{azar2012sample}. Hence, compared to the sample-complexity for unconstrained MDPs, the best-known upper-bounds for CMDPs are worse for both the relaxed and strict feasibility settings. However, 
it is unclear whether solving CMDPs is inherently more difficult than unconstrained MDPs. We resolve these questions for both the relaxed and strict feasibility settings, and make the following contributions. 

\textbf{Generic model-based algorithm}: In~\cref{sec:method}, we provide a generic model-based primal-dual algorithm (\cref{alg:cmdp-generative}) that can be used to achieve both the relaxed and strict feasibility objectives (with appropriate parameter settings). The proposed algorithm requires solving a sequence of unconstrained empirical MDPs using any black-box MDP planner. 

\textbf{Upper-bound on sample complexity under relaxed feasibility}: In~\cref{sec:ub-relaxed}, we prove that with a specific set of parameters,~\cref{alg:cmdp-generative} uses no more than $\tilde{O}\left(\frac{S A}{(1 - \gamma)^3 \epsilon^2}\right)$ samples to achieve the relaxed feasibility objective. This improves upon the bounds of~\citet{hasan2021model} and matches the lower-bound in the easier unconstrained MDP setting, implying that our bounds are near-optimal. Our result indicates that under relaxed feasibility solving CMDPs is as easy as solving unconstrained MDPs. To the best of our knowledge, these are the first such bounds.

\textbf{Upper-bound on sample-complexity under strict feasibility}: In~\cref{sec:ub-strict}, we prove that with a specific set of parameters,~\cref{alg:cmdp-generative} uses no more than $\tilde{O}\left(\frac{S A}{(1 - \gamma)^5 \zeta^2 \epsilon^2}\right)$ to achieve the strict feasibility objective. Here $\zeta \in \left(0, \nicefrac{1}{1 - \gamma} \right]$ is the problem-dependent \emph{Slater constant} that characterizes the size of the feasible region and influences the difficulty of the problem. Unlike~\citet{bai2021achieving}, our bounds do not depend on additional (potentially large) problem-dependent quantities. 

\textbf{Lower-bound on sample-complexity under strict feasibility}: In~\cref{sec:lb-strict}, we prove a matching  problem-dependent $\Omega \left(\frac{SA}{(1 - \gamma)^5 \, \zeta^2 \, \epsilon^2} \right)$ lower bound on the sample-complexity in the strict feasibility setting. Our results thus demonstrate that the proposed model-based algorithm is near minimax optimal. Furthermore, our bounds indicate that under strict feasibility (i) solving CMDPs is inherently more difficult than solving unconstrained MDPs, and (ii) the problem hardness (in terms of the sample-complexity) increases as $\zeta$ (and hence the size of the feasible region) decreases. To the best of our knowledge, these are first results characterizing the difficulty of solving CMDPs with access to a generative model and demonstrate a separation between the relaxed and strict feasibility settings. 

\textbf{Overview of techniques}: For proving the upper bounds, we use a specific primal-dual algorithm that reduces the CMDP planning problem to solving multiple unconstrained MDPs. Specifically, by using a strong-duality argument, we show that we can obtain an optimal CMDP policy by averaging the optimal policies of a specific sequence of MDPs. For each MDP in this sequence, we use the model-based techniques from~\citet{agarwal2020model, li2020breaking} to prove concentration results for data-dependent policies. This allows us to prove concentration for the optimal data-dependent policy in the CMDP, and subsequently bound the sample complexity for both the relaxed and strict feasibility problems. For the lower bound, we modify the MDP hard instances~\citep{azar2013minimax,xiao2021sample} to handle a constraint reward. This makes the resulting gadgets significantly more complex than those required for MDPs, but we show that similar likelihood arguments can be used to prove the lower-bound. 
\section{Problem Formulation}
\label{sec:problem}
We consider an infinite-horizon discounted constrained Markov decision process (CMDP)~\citep{altman1999constrained} denoted by $M$, and defined by the tuple $\langle \cS, \cA, \cP, r, c, b, \rho, \gamma \rangle$ where $\cS$ is the set of states, $\cA$ is the action set, $\cP : \cS \times \cA \rightarrow \Delta_\cS$ is the transition probability function, $\rho \in \Delta_{\cS}$ is the initial distribution of states and $\gamma \in [0, 1)$ is the discount factor. The primary reward to be maximized is denoted by $r : \cS \times \cA \rightarrow [0,1]$, whereas the constraint reward is denoted by $c: \cS \times \cA \rightarrow [0,1]$\footnote{These ranges for $r$ and $c$ are chosen for simplicity. Our results can be easily extended to handle other ranges.}. If $\Delta_{\cA}$ denotes the simplex over the action space, the expected discounted return or \emph{reward value function} of a stationary, stochastic policy\footnote{The performance of an optimal policy in a CMDP can always be achieved by a stationary, stochastic policy~\citep{altman1999constrained}. On the other hand, for an MDP, it suffices to only consider stationary, deterministic policies~\citep{puterman2014markov}.} $\pi: \cS \rightarrow \Delta_{\cA}$ is defined as $\reward{\pi} = \mathbb{E}_{s_0, a_0, \ldots} \Big[\sum_{t=0}^\infty \gamma^t r(s_t, a_t)\Big]$, where $s_0 \sim \rho, a_t \sim \pi( \cdot | s_t),$ and $s_{t+1} \sim \cP( \cdot | s_t, a_t)$. For each state-action pair $(s,a)$ and policy $\pi$, the reward action-value function is defined as $\rewardq{\pi}: \cS \times \cA \rightarrow \R$, and satisfies the relation: $V_{r}^{\pi}(s) = \langle \pi(\cdot | s), \rewardq{\pi}(s,\cdot) \rangle$, where $V_{r}^{\pi}(s)$ is the reward value function when the starting state is equal to $s$. Analogously, the \emph{constraint value function} and constraint action-value function of policy $\pi$ is denoted by $\const{\pi}$ and $\constq{\pi}$ respectively. The CMDP objective is to return a policy that maximizes $\reward{\pi}$, while ensuring that $\const{\pi} \geq b$. Formally, 
\begin{align}
\max_{\pi} \reward{\pi} \quad \text{s.t.} \quad  \const{\pi} \geq b.
\label{eq:true-CMDP}
\end{align}
The optimal stochastic policy for the above CMDP is denoted by $\piopt$ and the corresponding reward value function is denoted by $\reward{*}$. We also define $\zeta := \max_{\pi} \const{\pi} - b$ as the problem-dependent quantity referred to as the Slater constant~\citep{ding2021provably,bai2021achieving}. The Slater constant is a measure of the size of the feasible region and determines the difficulty of solving~\cref{eq:true-CMDP}. 

For simplicity of exposition, we assume that the rewards $r$ and constraint rewards $c$ are known, but the transition matrix $\cP$ is unknown and needs to be estimated. We note that assuming the knowledge of the rewards does not affect the leading terms of the sample complexity since learning these is an easier problem compared to the transition matrix~\citep{azar2013minimax, sidford2018near}. We assume access to a \emph{generative model} or simulator that allows the agent to obtain samples from the $\cP(\cdot|s,a)$ distribution for any $(s,a)$. Assuming access to such a generative model, our aim is to characterize the sample complexity required to return a near-optimal policy $\pihat$ in $M$. Given a target error $\epsilon > 0$, we can characterize the performance of policy $\pihat$ in two ways:

\textbf{Relaxed feasibility}: We require $\hat{\pi}$ to achieve an approximately optimal reward value, while allowing it to have a small constraint violation in $M$\footnote{In general, the desired gap in the reward value can be different from the level of constraint violation.}. Formally, we require $\pihat$ s.t. 
\begin{align}
\reward{\hat{\pi}} \geq \optreward - \epsilon, \text{ and } \const{\hat{\pi}} \geq b -\epsilon. 
\label{eq:relaxed-objective}    
\end{align}

\textbf{Strict feasibility} We require $\hat{\pi}$ to achieve an approximately optimal reward value, while simultaneously demanding zero constraint violation in $M$. Formally, we require $\pihat$ s.t. 
\begin{align}
\reward{\hat{\pi}} \geq \optreward - \epsilon, \text{ and } \const{\hat{\pi}} \geq b
\label{eq:strict-objective}    
\end{align}

Next, we describe a general model-based algorithm to handle both these cases, and subsequently instantiate the algorithm for the relaxed feasibility (\cref{sec:ub-relaxed}) and strict feasibility (\cref{sec:ub-strict}) settings.   
\section{Methodology}
\label{sec:method}
\begin{algorithm}[!t]
\caption{Model-based algorithm for CMDPs with generative model}
\label{alg:cmdp-generative}
\textbf{Input}: $\cS$ (state space), $\cA$ (action space), $r$ (rewards), $c$ (constraint rewards), $\zeta$ (Slater constant), $N$ (number of samples), $b'$ (constraint RHS), $\omega$ (perturbation magnitude), $U$ (projection upper bound), $\epsl$ (epsilon-net resolution), $T$ (number of iterations), $\lambda_0 = 0$ (initialization). \\

For each state-action $(s,a)$ pair, collect $N$ samples from $\cP(.|s,a)$ and form $\cPhat$. \\

Perturb the rewards to form vector $r_p(s,a) := r(s,a) + \xi(s,a)$ where $\xi(s,a) \sim \mathcal{U}[0, \omega]$. \\

Form the empirical CMDP $\Mhat = \langle \cS, \cA, \cPhat, r_p, c, b', \rho, \gamma \rangle$. \\

Form the epsilon-net $\Lambda = \{0, \epsl, 2 \epsl, \ldots, U\}$. \\

\For{$t \leftarrow 0$  \KwTo  $T-1$}{
Update the policy by solving an unconstrained MDP: $\pihatt = \argmax \hat{V}^{\pi}_{r_p + \lambdat c}$. \\
Update the dual-variables: $\lambdatt = \mathcal{R}_{\Lambda}\left[\PP_{[0,U]} \left[\lambdat - \eta \, (\consthat{\pihatt} - b')\right]\right]$. 
}
\textbf{Output}: Mixture policy $\bar{\pi}_T = \frac{1}{T} \sum_{t = 0}^{T-1} \pihatt$. 
\end{algorithm}

We will use a model-based approach~\citep{agarwal2020model,li2020breaking,hasan2021model} for achieving the objectives in~\cref{eq:relaxed-objective} and~\cref{eq:strict-objective}. In particular, for each $(s,a)$ pair, we collect $N$ independent samples from $\cP(\cdot|s,a)$ and form an empirical transition matrix $\cPhat$ such that $\cPhat(s'|s,a) = \frac{N(s'|s,a)}{N}$, where $N(s'|s,a)$ is the number of samples that have transitions from $(s,a)$ to $s'$. These estimated transition probabilities are used to form an empirical CMDP. Due to a technical requirement, (which we will clarify in the~\cref{sec:concentration}), we require adding a small random perturbation to the rewards in the empirical CMDP\footnote{Similar to MDPs~\citep{li2020breaking}, we can instead perturb the $Q$ function while planning in the empirical CMDP.}. In particular, for each $s \in \cS$ and $a \in \cA$, we define the perturbed rewards $r_p(s,a) := r(s,a) + \xi(s,a)$ where $\xi(s,a) \sim \mathcal{U}[0, \omega]$ are i.i.d. uniform random variables. Finally, compared to~\cref{eq:true-CMDP}, we will require solving the empirical CMDP with a constraint right-hand side equal to $b'$. Note that setting $b' < b$ corresponds to loosening the constraint, while $b' > b$ corresponds to tightening the constraint. This completes the specification of the empirical CMDP $\Mhat$ that is defined by the tuple $\langle \cS, \cA, \cPhat, r_p, c, b', \rho, \gamma \rangle$. For $\Mhat$, the corresponding reward value function (and constraint value function) for policy $\pi$ is denoted as $\rewardhatp{\pi}$ (and $\consthat{\pi}$ respectively). 
In order to fully instantiate $\Mhat$, we require setting the values of $\omega$ (the magnitude of the perturbation) and $b'$ (the constraint right-hand side). This depends on the specific setting (relaxed vs strict feasibility) and we do this in~\cref{sec:ub-relaxed,sec:ub-strict} respectively. We compute the optimal policy for the empirical CMDP $\Mhat$ as follows: 
\begin{align}
\pihatopt \in \argmax \rewardhatp{\pi} \, \text{s.t.} \, \consthat{\pi} \geq b'
\label{eq:emp-CMDP}
\end{align}

In contrast to~\citet{agarwal2020model,li2020breaking} that consider model-based approaches for unconstrained MDPs and can solve the resulting empirical MDP using any black-box approach, we will require solving~\cref{eq:emp-CMDP} using a specific primal-dual approach that we outline next. Using this algorithm enables us to prove optimal sample complexity bounds under both relaxed and strict feasibility. 

First, observe that~\cref{eq:emp-CMDP} can be written as an equivalent saddle-point problem -- $\max_{\pi} \min_{\lambda \geq 0} \left[\rewardhatp{\pi} + \lambda \left(\consthat{\pi} - b' \right) \right]$, where $\lambda \in \R$ corresponds to the Lagrange multiplier for the constraint. The solution to this saddle-point problem is $(\pihatopt, \lambda^*)$ where $\pihatopt$ is the optimal empirical policy and $\lambda^*$ is the optimal Lagrange multiplier. We solve the above saddle-point problem iteratively, by alternatively updating the policy (primal variable) and the Lagrange multiplier (dual variable). If $T$ is the total number of iterations of the primal-dual algorithm, we define $\pihatt$ and $\lambdat$ to be the primal and dual iterates for $t \in [T] :=\{1,\dots,T\}$. The primal update at iteration $t$ is given as: 
\begin{align}
\pihatt = \argmax \left[\rewardhatp{\pi} + \lambdat \consthat{\pi} \right] = \argmax \hat{V}^{\pi}_{r_p + \lambdat c}.  
\label{eq:primal-update} 
\end{align}
Hence, iteration $t$ of the algorithm requires solving an unconstrained MDP with a reward equal to $r_p + \lambdat c$. This can be done using any black-box MDP solver such as policy iteration. The algorithm updates the Lagrange multipliers using a gradient descent step and requires projecting and rounding the resulting dual variables. In particular, the dual variables are first projected onto the $[0,U]$ interval, where $U$ is chosen to be an upper-bound on $|\lambda^*|$. After the projection, the resulting iterates are rounded to the closest element in the set $\Lambda = \{0, \epsl, 2 \epsl, \ldots, U\}$, a one-dimensional epsilon-net (with resolution $\epsl$) over the dual variables. In~\cref{sec:concentration}, we will see that constructing such an $\epsl$-net will enable us to prove concentration results for all $\lambda \in \Lambda$. 

The dual update at iteration $t$ is given as:
\begin{align}
\lambdatt = \mathcal{R}_{\Lambda}\left[\PP_{[0,U]} \left[\lambdat - \eta \, (\consthat{\pihatt} - b')\right]\right] \, ,
\label{eq:dual-update}    
\end{align}
where $\PP_{[0,U]} [\lambda] = \argmin_{p \in [0,U]} \abs{\lambda - p}$ projects  $\lambda$ onto the $[0,U]$ interval

and $\mathcal{R}_{\Lambda}[\lambda] = \argmin_{p \in \Lambda } \abs{\lambda - p}$ rounds $\lambda$ to the closest element in $\Lambda$. Since $\Lambda$ is an epsilon-net, for all $\lambda \in [0,U]$, $\vert \lambda - \mathcal{R}_{\Lambda}[\lambda] \vert \leq \epsl$. Finally, $\eta$ in~\cref{eq:dual-update} corresponds to the step-size for the gradient descent update. The above primal-dual updates are similar to the dual-descent algorithm proposed in~\citet{paternain2019constrained}. The pseudo-code summarizing the entire model-based algorithm is given in~\cref{alg:cmdp-generative}. We note that although~\cref{alg:cmdp-generative} requires the knowledge of $\zeta$, this is not essential and we can instead use an estimate of $\zeta$. In~\cref{app:est-zeta}, we show that we can estimate $\zeta$ to within a factor of 2 using $\tilde{O}\left(\frac{|\cS||\cA|}{(1-\gamma)^3\zeta^2}\right)$ additional queries. Next, we show that the primal-dual updates in~\cref{alg:cmdp-generative} can be used to solve the empirical CMDP $\Mhat$. Specifically, we prove the following theorem (proof in~\cref{app:proofs-pd}) that bounds the average optimality gap (in the reward value function) and constraint violation for the mixture policy returned by~\cref{alg:cmdp-generative}.
\begin{restatable}[Guarantees for the primal-dual algorithm]{theorem}{pdguarantees}
\label{thm:pd-guarantees}
For a target error $\epsp > 0$ and the primal-dual updates in~\cref{eq:primal-update}-\cref{eq:dual-update} with $U > |\lambda^*|$, $T = \frac{4 U^2}{\epsp^2 \, (1 - \gamma)^2} \left[1 + \frac{1}{(U - \lambda^*)^2} \right]$, $\eta = \frac{U (1 - \gamma)}{\sqrt{T}}$ and $\epsl = \frac{\epsp^2 (1 - \gamma)^2 \, (U - \lambda^*)}{6 U}$, the mixture policy $\pihatbar := \frac{1}{T} \sum_{t = 0}^{T-1} \pihatt$ satisfies
\begin{align*}
\rewardhatp{\pihatbar} & \geq \rewardhatp{\pihatopt} - \epsp \quad \text{and} \quad \consthat{\pihatbar} \geq b' - \epsp\,. 
\end{align*}
\end{restatable}
Hence, with $T = O(\nicefrac{1}{\epsp^2})$ and $\epsl = O(\epsp^2)$, the algorithm outputs a policy $\pihatbar$ that achieves a reward $\epsp$ close to that of the optimal empirical policy $\pihatopt$, while violating the constraint by at most $\epsp$. Hence, with sufficient number of iterations $T$ and by choosing a sufficiently small resolution $\epsl$ for the epsilon-net, we can use the above primal-dual algorithm to approximately solve the problem in~\cref{eq:emp-CMDP}. In order to completely instantiate the primal-dual algorithm, we require setting $U > |\lambda^*|$. We will subsequently do this for the the relaxed and strict feasibility settings in~\cref{sec:ub-relaxed,sec:ub-strict} respectively. We note that in contrast to~\citet[Theorem 3]{paternain2019constrained} that bounds the Lagrangian,~\cref{thm:pd-guarantees} provides explicit bounds on both the reward suboptimality and constraint violation. 

We conclude this section by making some observations about the primal-dual algorithm -- while the subsequent bounds for both settings heavily depend on using the ``best-response'' primal update in~\cref{eq:primal-update}, the algorithm does not require using the specific form of the dual updates in~\cref{eq:dual-update}. Indeed, when used in conjunction with the projection and rounding operations in~\cref{eq:dual-update}, we can use any method to update the dual variables (not necessarily gradient descent) provided that it results in an $O\left(T^{a} + \epsl T^{b} \right)$ (for $a < 1$) bound on the dual regret (see the proof of~\cref{thm:pd-guarantees} for the definition). Next, we specify the values of $N, b'$, $\omega$, $\epsl$, $T$, $U$ in~\cref{alg:cmdp-generative} to achieve the objective in~\cref{eq:relaxed-objective}. 
 
\section{Upper-bound under Relaxed Feasibility}
\label{sec:ub-relaxed}
In order to achieve the objective in~\cref{eq:relaxed-objective} for a target error $\epsilon > 0$, we require setting $N = \tilde{O}\left(\frac{\log(1/\delta)}{(1 - \gamma)^3 \epsilon^2} \right)$, $b' = b - \frac{3 \epsilon}{8}$ and $\omega = \frac{\epsilon (1 - \gamma)}{8}$. This completely specifies the empirical CMDP $\Mhat$ and the problem in~\cref{eq:emp-CMDP}. In order to specify the primal-dual algorithm, we set $U = O\left(\nicefrac{1}{\epsilon \, (1 - \gamma)}\right)$, $\epsl = O\left(\epsilon^2 (1 - \gamma)^2 \right)$ and $T = O\left(\nicefrac{1}{(1 - \gamma)^4 \epsilon^4}\right)$. With these choices, we prove the following theorem in~\cref{app:proof-relaxed} and provide a proof sketch below. 
\begin{restatable}{theorem}{ubrelaxed}
For a fixed $\epsilon \in \left(0, \nicefrac{1}{1 - \gamma}\right]$ and  $\delta\in(0,1)$,  \cref{alg:cmdp-generative} with $N = \tilde{O}\left(\frac{\log(1/\delta)}{(1 - \gamma)^3 \epsilon^2} \right)$ samples, $b' = b - \frac{3 \epsilon}{8}$, $\omega = \frac{\epsilon (1 - \gamma)}{8}$, $U = O\left(\nicefrac{1}{\epsilon \, (1 - \gamma)}\right)$, $\epsl = O\left(\epsilon^2 (1 - \gamma)^2 \right)$ and $T = O\left(\nicefrac{1}{(1 - \gamma)^4 \epsilon^4}\right)$, returns policy $\pihatbar$  that satisfies the objective in~\cref{eq:relaxed-objective} with probability at least $1 - 4 \delta$.
\label{thm:ub-relaxed}
\end{restatable}
\begin{proofsketch}
We prove the result for a general primal-dual error $\epsp < \epsilon$ and $b' = b - \frac{\epsilon - \epsp}{2}$, and subsequently specify $\epsp$ and hence $b'$. 
In~\cref{lemma:decomposition-relaxed} (proved in~\cref{app:proof-relaxed}), we show that if the constraint value functions are sufficiently concentrated (the empirical value function is close to the ground truth value function) for both the optimal policy $\piopt$ in $M$ and the mixture policy $\pihatbar$ returned by~\cref{alg:cmdp-generative}, i.e., if  
\begin{align}
\abs{\const{\pihatbar} - \consthat{\pihatbar}} \leq \frac{\epsilon - \epsp}{2} \quad \text{;} \quad \abs{\const{\piopt} - \consthat{\piopt}} \leq \frac{\epsilon - \epsp}{2}, \label{eq:sampling-const-relaxed} 
\end{align}
then (i) policy $\pihatbar$ violates the constraint in $M$ by at most $\epsilon$, i.e., $\const{\pihatbar} \geq b - \epsilon$, and (ii) its suboptimality in $M$ (compared to $\piopt$) can be decomposed as:
\begin{align}
\reward{\piopt} - \reward{\pihatbar} & \leq \frac{2 \omega}{1 - \gamma} + \epsp + \abs{\rewardp{\piopt} - \rewardhatp{\piopt}} + \abs{\rewardhatp{\pihatbar} - \rewardp{\pihatbar}}.
\label{eq:decomposition-relaxed} 
\end{align}

In order to instantiate the primal-dual algorithm, we require a concentration result for policy $\pi^*_c$ that maximizes the the constraint value function, i.e. if $\pi^*_c := \argmax \const{\pi}$, then we require $\abs{\const{\pi^*_c} - \consthat{\pi^*_c}} \leq \epsilon + \epsp$. In Case 1 of~\cref{lemma:dual-bound} (proved in~\cref{app:proofs-pd}), we show that if this concentration result holds, then we can upper-bound the optimal dual variable $|\lambda^*|$ by $\frac{2 (1 + \omega)}{(\epsilon + \epsp) (1 - \gamma)}$. With these results in hand, we can instantiate all the algorithm parameters except $N$ (the number of samples required for each state-action pair). In particular, we set $\epsp = \frac{\epsilon}{4}$ and hence $b' = b - \frac{3 \epsilon}{8}$, and $\omega = \frac{\epsilon (1 - \gamma)}{8} < 1$. Setting $U = \frac{32}{5 \epsilon \, (1 - \gamma)}$ ensures that the $U > |\lambda^*|$ condition required by~\cref{thm:pd-guarantees} holds. To guarantee that the primal-dual algorithm outputs an $\frac{\epsilon}{4}$-approximate policy, we use~\cref{thm:pd-guarantees} to set $T = O\left(\frac{1}{(1 - \gamma)^4 \epsilon^4}\right)$ iterations and $\epsl = O\left(\epsilon^2 (1 - \gamma)^2 \right)$.~\cref{eq:decomposition-relaxed} can then be simplified as,
\begin{align}
\reward{\piopt} - \reward{\pihatbar} & \leq \frac{\epsilon}{2} + \abs{\rewardp{\piopt} - \rewardhatp{\piopt}} + \abs{\rewardhatp{\pihatbar} - \rewardp{\pihatbar}}. \nonumber
\end{align}
Putting everything together, in order to guarantee an $\epsilon$-reward suboptimality for $\pihatbar$, we require that:
\begin{align}
\abs{\const{\pi^*_c} - \consthat{\pi^*_c}} & \leq \frac{5 \epsilon}{4} \, \text{;} \, \abs{\const{\pihatbar} - \consthat{\pihatbar}} \leq \frac{3 \epsilon}{8} \, \text{;} \, \abs{\const{\piopt} - \consthat{\piopt}} \leq \frac{3 \epsilon}{8} \nonumber \\
\abs{\rewardp{\piopt} - \rewardhatp{\piopt}} & \leq \frac{\epsilon}{4} \, \text{;} \, \abs{\rewardhatp{\pihatbar} - \rewardp{\pihatbar}} \leq \frac{\epsilon}{4} \label{eq:conc-relaxed}. 
\end{align}
We control such concentration terms for both the constraint and reward value functions in~\cref{sec:concentration}, and bound the terms in~\cref{eq:conc-relaxed}. In particular, we prove that for a fixed $\epsilon \in \left(0, \nicefrac{1}{1 - \gamma}\right]$, using $N \geq \tilde{O}\left(\frac{\log(1/\delta)}{(1 - \gamma)^3 \, \epsilon^2}\right)$ samples enssures that the statements in~\cref{eq:conc-relaxed} hold  with probability $1 - 4 \delta$. This guarantees that $\reward{\piopt} - \reward{\pihatbar} \leq \epsilon$ and $\const{\pihatbar} \geq b - \epsilon$.
\end{proofsketch}
Hence, the total sample-complexity of achieving the objective in~\cref{eq:relaxed-objective} is $\tilde{O}\left(\frac{S A \log(1/\delta)}{(1 - \gamma)^3 \epsilon^2} \right)$. This result improves over the $\tilde{O}\left(\frac{S^2 A \log(1/\delta)}{(1 - \gamma)^3 \epsilon^2} \right)$ result in~\citet{hasan2021model}. Furthermore, our result matches the lower-bound in the easier unconstrained setting~\citep{azar2012sample}, implying that our bounds are near-optimal. We conclude that under relaxed feasibility and with access to a generative model, solving constrained MDPs is as easy as solving MDPs. Algorithmically, we do not require constructing an optimistic CMDP like in~\citet{hasan2021model}. Instead, we solve the empirical CMDP in~\cref{eq:emp-CMDP} using specific primal-dual updates~\cref{eq:primal-update,eq:dual-update}. Note that if the rewards and constraint rewards (corresponding to $K$ constraints) are unknown and need to be estimated, a union bound guarantees that the sample complexity will only increase by a multiplicative $\log(K+1)$ factor~\citep{hasan2021model}.

In this setting, when using the \emph{normalized value functions},~\citet[Corollary 1]{bai2021achieving} prove an $O\left(\frac{S A}{(1 - \gamma)^2 \, \epsilon^2 \zeta^2}\right)$ bound on the sample complexity. When translated to the standard $\left[0,\nicefrac{1}{1 - \gamma}\right]$ range, this implies an $O\left(\frac{S A}{(1 - \gamma)^4 \,\epsilon^2 \zeta^2}\right)$ bound~\citep[Footnote 6]{bai2021achieving}. In comparison, our result in~\cref{thm:ub-relaxed} has a better dependence on $\nicefrac{1}{1 - \gamma}$ and does not depend on $\zeta$. Importantly, unlike~\citep{bai2021achieving}, our result implies that in the relaxed feasibility setting, solving CMDPs is as hard as solving MDPs. 

In the next section, we instantiate~\cref{alg:cmdp-generative} in the strict feasibility setting. 

\section{Upper-bound under Strict Feasibility}
\label{sec:ub-strict}
Unlike~\cref{sec:ub-relaxed}, since the strict feasibility setting does not allow any constraint violations, it necessitates using a stricter constraint in the empirical CMDP to account for the estimation error in the transition probabilities. Algorithmically, we require setting $b' > b$. Specifically, in order to achieve the objective in~\cref{eq:strict-objective} for a target error $\epsilon > 0$, we require setting $N = \tilde{O}\left(\frac{\log(1/\delta)}{(1 - \gamma)^5 \zeta^2 \epsilon^2} \right)$,\footnote{Again, we do not need to know $\zeta$ and it can be replaced by the estimator constructed in Section~\ref{app:est-zeta}.} $b' = b + \frac{\epsilon (1 - \gamma) \zeta}{20}$ and $\omega = \frac{\epsilon (1 - \gamma)}{10}$. This completely specifies the empirical CMDP $\Mhat$ and the problem in~\cref{eq:emp-CMDP}. To specify the primal-dual algorithm, we set $U = \frac{4 (1 + \omega)}{\zeta (1 - \gamma)}$,  $\epsl = O\left(\epsilon^2 (1 - \gamma)^4 \zeta^2 \right)$ and $T = O\left(\nicefrac{1}{(1 - \gamma)^6 \zeta^4 \epsilon^2}\right)$. With these choices, we prove the following theorem in~\cref{app:proof-strict}, and provide a proof sketch below. 
\begin{restatable}{theorem}{ubstrict}
For a fixed $\epsilon \in \left(0, \nicefrac{1}{1 - \gamma}\right]$ and  $\delta\in(0,1)$, \cref{alg:cmdp-generative},  with $N = \tilde{O}\left(\frac{\log(1/\delta)}{(1 - \gamma)^5 \epsilon^2 \zeta^2} \right)$ samples, $b' = b + \frac{\epsilon (1 - \gamma) \zeta}{20}$, $\omega = \frac{\epsilon (1 - \gamma)}{10}$, $U = \frac{4 (1 + \omega)}{\zeta (1 - \gamma)}$, $\epsl = O\left(\epsilon^2 (1 - \gamma)^4 \zeta^2 \right)$ and $T = O\left(\nicefrac{1}{(1 - \gamma)^6 \zeta^4 \epsilon^2}\right)$ returns policy $\pihatbar$ that satisfies the objective in~\cref{eq:strict-objective}, with probability at least $1 - 4 \delta$.
\label{thm:ub-strict} 
\end{restatable}
\begin{proofsketch}
We prove the result for a general $b' = b + \Delta$ for $\Delta > 0$ and primal-dual error $\epsp < \Delta$, and subsequently specify $\Delta$ (and hence $b'$) and $\epsp$. In~\cref{lemma:decomposition-strict} (proved in~\cref{app:proof-strict}), we prove that if the constraint value functions are sufficiently concentrated (the empirical value function is close to the ground truth value function) for both the optimal policy $\piopt$ in $M$ and the mixture policy $\pihatbar$ returned by~\cref{alg:cmdp-generative} i.e. if 
\begin{align}
\abs{\const{\pihatbar} - \consthat{\pihatbar}} \leq \Delta - \epsp \quad \text{;} \quad \abs{\const{\piopt} - \consthat{\piopt}} \leq \Delta
\label{eq:sampling-const-strict}
\end{align}
then (i) policy $\pihatbar$ satisfies the constraint in $M$ i.e. $\const{\pihatbar} \geq b$, and (ii) its suboptimality in $M$ (compared to $\piopt$) can be decomposed as:
\begin{align}
\reward{\piopt} - \reward{\pihatbar} & \leq \frac{2 \omega}{1 - \gamma} + \epsp + 2 \Delta |\lambda^*| + \abs{\rewardp{\piopt} - \rewardhatp{\piopt}} +  \abs{\rewardhatp{\pihatbar} - \rewardp{\pihatbar}}
\label{eq:decomposition-strict}
\end{align}
In order to upper-bound $|\lambda^*|$, we require a concentration result for policy $\pi^*_c := \argmax \const{\pi}$ that maximizes the the constraint value function. In particular, we require $\Delta \in \left(0,\frac{\zeta}{2}\right)$ and $\abs{\const{\pi^*_c} - \consthat{\pi^*_c}} \leq \frac{\zeta}{2} - \Delta$. In Case 2 of~\cref{lemma:dual-bound} (proved in~\cref{app:proofs-pd}), we show that if this concentration result holds, then we can upper-bound the optimal dual variable $|\lambda^*|$ by $\frac{2 (1 + \omega)}{\zeta (1 - \gamma)}$. Using the above bounds to simplify~\cref{eq:decomposition-strict},
\begin{align}
\reward{\piopt} - \reward{\pihatbar} & \leq \frac{2 \omega}{1 - \gamma} + \epsp + \frac{4 \Delta (1 + \omega)}{\zeta (1 - \gamma)} + \abs{\rewardp{\piopt} - \rewardhatp{\piopt}} +  \abs{\rewardhatp{\pihatbar} - \rewardp{\pihatbar}} \nonumber.
\end{align}
With these results in hand, we can instantiate all the algorithm parameters except $N$ (the number of samples required for each state-action pair). In particular, we set $\Delta = \frac{\epsilon \, (1 - \gamma) \, \zeta}{40} < \frac{\zeta}{2}$, $\epsp = \frac{\Delta}{5} = \frac{\epsilon \, (1 - \gamma) \, \zeta}{200} < \frac{\epsilon}{5}$, 
and $\omega = \frac{\epsilon (1 - \gamma)}{10} < 1$. We set $U = \frac{8}{\zeta (1 - \gamma)}$ for the primal-dual algorithm, ensuring that the $U > |\lambda^*|$ condition required by~\cref{thm:pd-guarantees} holds. In order to guarantee that the primal-dual algorithm outputs an $\frac{\epsilon \, (1 - \gamma) \, \zeta}{200}$-approximate policy, we use~\cref{thm:pd-guarantees} to set $T = O\left(\frac{1}{(1 - \gamma)^6 \zeta^4 \epsilon^2}\right)$ iterations and $\epsl = O\left(\epsilon^2 (1 - \gamma)^4 \zeta^2 \right)$. With these values, we can further simplify~\cref{eq:decomposition-strict},
\begin{align}
\reward{\piopt} - \reward{\pihatbar} & \leq \frac{3 \epsilon}{5} + \abs{\rewardp{\piopt} - \rewardhatp{\piopt}} +  \abs{\rewardhatp{\pihatbar} - \rewardp{\pihatbar}}. \nonumber
\end{align}
Putting everything together, in order to guarantee an $\epsilon$-reward suboptimality for $\pihatbar$, we require the following concentration results to hold for $\Delta = \frac{\epsilon (1 - \gamma) \zeta}{40}$,
\begin{align}
\abs{\const{\pihatbar} - \consthat{\pihatbar}} & \leq \frac{4 \Delta}{5} \, \text{;} \, \abs{\const{\piopt} - \consthat{\piopt}} \leq \Delta \, \text{;} \abs{\const{\pi^*_c} - \consthat{\pi^*_c}} \leq \frac{19 \Delta}{5} \nonumber \\
\abs{\rewardp{\piopt} - \rewardhatp{\piopt}} & \leq \frac{\epsilon}{5} \, \text{;} \, \abs{\rewardhatp{\pihatbar} - \rewardp{\pihatbar}} \leq \frac{\epsilon}{5}. \label{eq:conc-strict} 
\end{align}
We control such concentration terms for both the constraint and reward value functions in~\cref{sec:concentration}, and bound the terms in~\cref{eq:conc-strict}. In particular, we prove that for a fixed $\epsilon \in \left(0, \nicefrac{1}{1 - \gamma}\right]$, using $N \geq \tilde{O}\left(\frac{\log(1/\delta)}{(1 - \gamma)^5 \, \zeta^2 \, \epsilon^2}\right)$ ensures that the statements in~\cref{eq:conc-strict} hold with probability $1 - 4 \delta$. This guarantees that $\reward{\piopt} - \reward{\pihatbar} \leq \epsilon$ and $\const{\pihatbar} \geq b$. 
\end{proofsketch}
\vspace{-2ex}
Hence, the total sample-complexity of achieving the objective in~\cref{eq:strict-objective} is $ \tilde{O} \left( \frac{S A \, \log(1/\delta)}{(1 - \gamma)^5 \, \zeta^2 \epsilon^2} \right)$. Similar to~\cref{sec:ub-relaxed}, in the strict feasibility setting, with the \emph{normalized value functions}, ~\citet{bai2021achieving} prove an $O\left(\frac{S A}{(1 - \gamma)^2 \, \epsilon^2 \zeta^2}\right)$ bound on the sample complexity. When translated to the standard $\left[0, \nicefrac{1}{1 - \gamma}\right]$ range, this implies an $\Omega\left(\frac{S A}{(1 - \gamma)^6 \, \epsilon^2 \zeta^2}\right)$ bound (see~\cref{app:bai-comparison} for a detailed explanation). In comparison, our result in~\cref{thm:ub-strict} has a better dependence on $(\nicefrac{1}{1 - \gamma})$.

In~\cref{sec:lb-strict}, we prove a matching lower bound showing that~\cref{alg:cmdp-generative} is minimax optimal in the strict feasibility setting. In the next section, we give more details for the bounding the concentration terms in~\cref{thm:ub-relaxed} and~\cref{thm:ub-strict}. 
\section{Bounding the concentration terms}
\label{sec:concentration}
We have seen that proving~\cref{thm:ub-relaxed} and~\cref{thm:ub-strict} require bounding the concentration terms in~\cref{eq:conc-relaxed} and~\cref{eq:conc-strict} respectively. In this section, we detail the techniques to achieve these bounds. 

Our approach requires reasoning about a general unconstrained MDP $M_{\alpha} = (\cS, \cA, \cP, \gamma, \alpha)$ with the same state-action space, transition probabilities and discount factor as the CMDP in~\cref{eq:true-CMDP} but with rewards equal to $\alpha$, coming from $[0, \alpha_{\max}]$. Analogously, we define the empirical MDP $\hat{M}_{\alpha} = (\cS, \cA, \cPhat, \gamma, \alpha)$ where the empirical transition matrix $\cPhat$ is the same as that of the empirical CMDP in~\cref{eq:emp-CMDP}. Similarly, we define MDP (and its empirical counterpart) $M_{\beta} = (\cS, \cA, \cP, \gamma, \beta)$ (and $\hat{M}_{\beta}$) where the rewards $\beta$ are from $[0, \beta_{\max}]$. Note that the rewards $\alpha$ and $\beta$ are independent of the sampling of the transition matrix. The corresponding value functions for policy $\pi$ in $M_\alpha$ and $\hat{M}_\alpha$ (and $M_\beta$ and $\hat{M}_\beta$) are denoted as $\rewarda{\pi}$ and $\rewardahat{\pi}$ (and $\rewardb{\pi}$ and $\rewardbhat{\pi}$) respectively, with the optimal value functions denoted as $\rewarda{*}$ and $\rewardahat{*}$ (and $\rewardb{*}$ and $\rewardbhat{*}$) respectively. The action-value function in $M_\alpha$ for policy $\pi$ and state-action pair $(s,a)$ is denoted as ${Q}^\pi_{\alpha}(s,a)$ and analogously for $\hat{M}_\alpha$. For the subsequent technical results, we require that $\hat{M}_\alpha$ satisfy the following gap condition~\citep{li2020breaking}: 
\begin{definition}[$\iota$-Gap Condition] MDP $\hat{M}_\alpha$ satisfies the $\iota$-gap condition if $\forall s$, $\hat{V}^*_{\alpha}(s) - \max_{a': a \neq \pihat^*_{\alpha}(s)}\hat{Q}^*_{\alpha}(s,a') \ge \iota$, where $\pihat^*_{\alpha} := \argmax \hat{V}^\pi_{\alpha}$ and $\pihat^*_{\alpha}(s) = \argmax_{a} \hat{Q}^*_{\alpha}(s,a)$ is the optimal action in state $s$. 
\label{def:gap-condition}
\end{definition}
Intuitively, the gap condition states that there is a unique optimal action at each state and there is a gap between the performance of best action and the second best action. With this gap condition, we use techniques in~\citet{li2020breaking} to prove the following lemma in~\cref{app:proofs-concentration}. 
\begin{restatable}{lemma}{main}
\label{lemma:main} 
Define $\pihat^*_\alpha := \argmax_{\pi} \rewardahat{\pi}$. If (i) $\cE$ is the event that the $\iota$-gap condition in~\cref{def:gap-condition} holds for $\hat{M}_\alpha$ and (ii) for $\delta \in (0,1)$ and $C(\delta) = 72 \log \left( \frac{16 \alpha_{\max} S A \log\left(\nicefrac{e}{1 - \gamma}\right)}{(1 - \gamma)^2 \, \iota \, \delta} \right)$, the number of samples per state-action pair is $N \geq \frac{4 \, C(\delta)}{1-\gamma}$, then with probability at least $\Pr[\cE] - \delta/10$,
\[
\norminf{\rewardbhat{\pihat^*_\alpha}  - \rewardb{\pihat^*_\alpha}} \leq  \sqrt{
\frac{C(\delta)}{N\cdot (1-\gamma)^3 }} \norminf{\beta}.
\]
\end{restatable}
\vspace{-1ex}
Hence, for policy $\pihat^*_\alpha$, we can obtain a concentration result in another MDP $\hat{M}_{\beta}$ with an independent reward function $\beta$ and the same empirical transition matrix $\cPhat$. 

We wish to use the above lemma for the unconstrained MDP formed at every iteration of the primal update in~\cref{eq:primal-update}. In particular, for a given $\lambdat$, we will use~\cref{lemma:main} with $\alpha = r_p + \lambdat c$ and $\beta = r_p$. Doing so will immediately give us a bound on $\norminf{\rewardrp{\pihatt} - \rewardrphat{\pihatt}}$ and hence $\abs{\rewardp{\pihatt} - \rewardhatp{\pihatt}}$. In order to use~\cref{lemma:main}, we require the unconstrained MDP $\hat{M}_{r_p + \lambda c}$ to satisfy the gap condition in~\cref{def:gap-condition} for any $\lambda \in \Lambda$. This is achieved by the perturbation of the rewards in Line 3 of~\cref{alg:cmdp-generative}. Specifically, using~\citet[Lemma 6]{li2020breaking} with a union-bound over $\Lambda$, we prove (in~\cref{lemma:gap} in~\cref{app:proofs-concentration}) that with probability $1 - \delta/10$, $\hat{M}_{r_p + \lambda c}$ satisfies the gap condition in~\cref{def:gap-condition} with $\iota = \frac{\omega \, \delta \, (1-\gamma)}{30 \, |\Lambda||S||A|^2}$ for every $\lambda \in \Lambda$. This allows us to use~\cref{lemma:main} with $\alpha = r_p + \lambdat c$ for all $t \in [T]$, and $\beta = r_p$ and $\beta = c$. In the following theorem, we obtain a concentration result for each $\pihatt$ and hence for the mixture policy $\pihatbar$.  
\begin{restatable}{theorem}{mainconcentrationemp}
\label{thm:main-concentration-emp}
For $\delta\in(0,1)$, $\omega \leq 1$ and $C(\delta) = 72 \log \left(\frac{16 (1 + U + \omega) \, S A \log\left(\nicefrac{e}{1 - \gamma}\right)}{(1 - \gamma)^2 \, \iota \, \delta} \right)$ where $\iota = \frac{\omega \, \delta \, (1-\gamma) \, \epsl}{30 \, U |S||A|^2}$, if $N \geq \frac{4 \, C(\delta)}{1-\gamma}$, then for $\pihatbar$ output by~\cref{alg:cmdp-generative}, with probability at least $1 - \delta/5$, 
\[
\abs{\rewardp{\pihatbar} - \rewardhatp{\pihatbar}} \leq  2 \sqrt{
\frac{C(\delta)}{N \cdot (1-\gamma)^3 }} \quad \text{;} \quad 
\abs{\const{\pihatbar} - \const{\pihatbar}} \leq  \sqrt{
\frac{C(\delta)}{N \cdot (1-\gamma)^3 }}.
\]
\end{restatable} 
\cref{eq:conc-relaxed,eq:conc-strict} also require proving concentration bounds for fixed (that do not depend on the data) policies $\piopt$ and $\pi^*_c$. This can be done by directly using~\citet[Lemma 1]{li2020breaking}. Specifically, we prove the following the lemma in~\cref{app:proofs-concentration}. 
\begin{restatable}{lemma}{mainconcentrationopt}
\label{lemma:main-concentration-opt}
For $\delta \in (0,1)$, $\omega \leq 1$ and $C'(\delta) = 72 \log \left(\frac{4 |S| \log(e/1-\gamma)}{\delta}\right)$, if $N \geq \frac{4 \, C'(\delta)}{1 - \gamma}$ and $B(\delta,N) := \sqrt{\frac{C'(\delta)}{(1 - \gamma)^3 N}}$, then with probability at least $1 - 3 \delta$, 
\begin{align*}
\abs{\rewardp{\piopt} - \rewardhatp{\piopt}} & \leq  2  B(\delta,N) \, \text{;} \, \abs{\const{\piopt} - \consthat{\piopt}} \leq B(\delta,N) \, \text{;} \, \abs{\const{\pi^*_c} - \consthat{\pi^*_c}} \leq B(\delta,N). 
\end{align*}
\end{restatable}
Using~\cref{thm:main-concentration-emp} and~\cref{lemma:main-concentration-opt}, we can bound each term in~\cref{eq:conc-relaxed} and~\cref{eq:conc-strict}, completing the proof of~\cref{thm:ub-relaxed} and~\cref{thm:ub-strict} respectively. In the next section, we prove a lower-bound on the sample-complexity in the strict feasibility setting. 

\section{Lower-bound under strict feasibility}
\label{sec:lb-strict}
For a target error of $\epsilon$, our lower bound construction demonstrates that it is important to estimate the constraint value function to a smaller error equal to $\epsilon' := \epsilon (1 - \gamma) \zeta$. Intuitively, this is because a small ($\epsilon'$) estimation error in the constraint value can incorrectly render the optimal policy infeasible and result in a large $\nicefrac{\epsilon'}{(1 - \gamma) \zeta}$ suboptimality in the reward value. In~\cref{app:lb-bandit}, we detail this intuition in a simplified bandit setting and present the formal CMDP lower-bound below.

\begin{figure}[!ht]
\begin{center}
\includegraphics[width=0.8\textwidth]{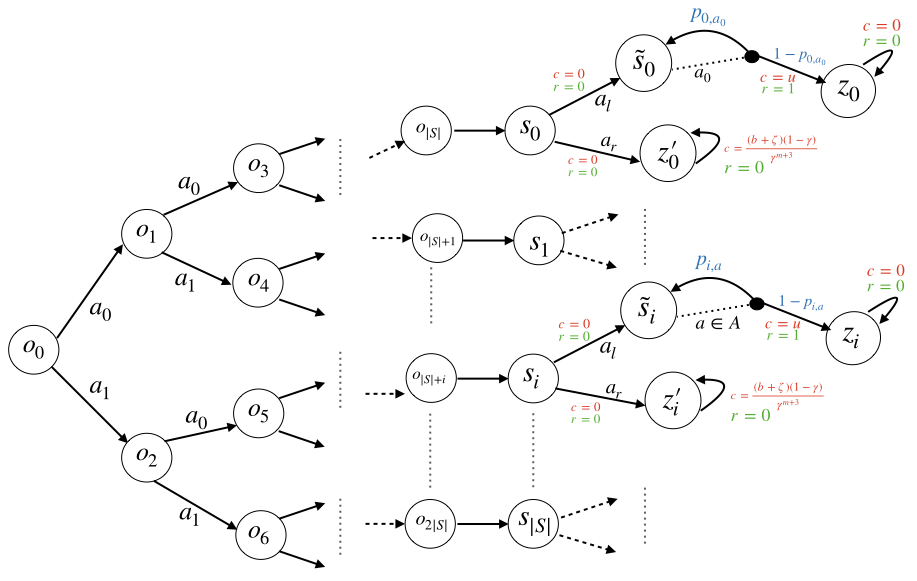}    
\end{center}
\caption{The instance consists of CMDPs with $S = 2^m - 1$ (for some integer $m > 0$) states and $A$ actions. We consider $SA + 1$ CMDPs -- $M_0$ and $M_{i,a}$ ($i \in \{1,\ldots S\}$, $a \in \{1,\ldots, A\}$) that share the same structure shown in the figure. For each CMDP, $o_0$ is the fixed starting state and there is a deterministic path of length $m+1$ from $o_0$ to each of the $S + 1$ states -- $s_i$ (for $i \in \{0, 1, \ldots, S\}$). Except for states $\tilde{s}_i$, the transitions in all other states are deterministic. For $i \neq 0$, for action $a \in \cA$ in state $\tilde{s}_i$, the probability of staying in $\tilde{s}_i$ is $p_{i,a}$, while that of transitioning to state $z_i$ is $1 - p_{i,a}$. There is only one action $a_0$ in $\tilde{s}_0$ and the probability of staying in $\tilde{s}_0$ is $p_{0,a_0}$, while that of transitioning to state $z_0$ is $1 - p_{0,a_0}$. The CMDPs $M_0$ and $M_{i,a}$ only differ in the values of $p_{i,a}$. The rewards $r$ and constraint rewards $c$ are the same in all CMDPs and are denoted in green and red respectively.
    } 
    \label{fig:lb-main}
\end{figure}
\vspace{1ex}
We define an algorithm to be $(\epsilon, \delta)$-sound if it outputs a policy $\pihat$ such that with probability $1 - \delta$, $V_r^{*}(\rho) - V_r^{\pihat}(\rho) \leq \epsilon$ and $V_c^{\pihat}(\rho) \geq b$ i.e. the algorithm achieves the strict feasibility objective in~\cref{eq:strict-objective}. We prove a lower bound on the number of samples required by any $(\epsilon, \delta)$-sound algorithm on the CMDP instance in~\cref{fig:lb-main}. For this instance, with a specific setting of the rewards and probabilities $p_0 < \bar{p} < p_1$, we prove that any $(\epsilon, \delta)$-sound algorithm requires at least $\Omega \left(\frac{\ln(|\cS| |\cA|/4 \delta)}{\epsilon^2 \zeta^2 (1- \gamma)^5} \right)$ samples to distinguish between $M_0$ and $M_{i,a}$. In particular, we prove the following theorem in~\cref{app:proof-lb}. 

\begin{restatable}{theorem}{lbstrict}
There exists constants $\gamma_0 \in (1-1/\log(|\cS|),1)$, 
$0\le \epsilon_0 \le \frac{1}{(1 - \gamma)} \, \min\left\{1, \frac{\gamma}{(1 - \gamma) \zeta} \right\}$, $\delta_0\in (0, 1)$, such that, 
for any $\gamma \in (\gamma_0, 1), \epsilon \in (0, \epsilon_0), \delta \in (0,\delta_0)$, any $(\epsilon, \delta)$-sound algorithm requires $\Omega \left(\frac{SA \ln(1/4 \delta)}{\epsilon^2 \zeta^2 (1- \gamma)^5} \right)$ samples from the generative model in the worst case.
\label{thm:lb-strict}
\end{restatable}
The above lower bound matches the upper bound in~\cref{thm:ub-strict} and proves that~\cref{alg:cmdp-generative} is near minimax optimal in the strict feasibility setting. It also demonstrates that solving CMDPs under strict feasibility is inherently more difficult than solving unconstrained MDPs or CMDPs in the relaxed feasibility setting. Finally, we can conclude that the problem becomes more difficult (requires more samples) as the Slater constant $\zeta$ decreases and the feasible region shrinks. 
\section{Discussion}
\label{sec:discussion}
We proposed a model-based primal-dual algorithm for planning in CMDPs. Via upper and lower bounds, we proved that our algorithm is near minimax optimal for both the relaxed and strict feasibility settings. Our results demonstrate that solving CMDPs is as easy as MDPs when small constraint violations are allowed, but inherently more difficult when we demand zero constraint violation. Algorithmically, we required a specific primal-dual approach that involved solving a sequence of MDPs. In contrast, model-based approaches for MDPs~\citep{agarwal2020model,li2020breaking} allow the use of any black-box planner. It is possible to obtain an $O\left(\frac{S^2 A}{(1 - \gamma)^3 \, \epsilon^2}\right)$ sample complexity for a black-box CMDP planner in the relaxed feasibility setting. However, the $O(S^2 A)$ dependence in the bound implies that we need to accurately estimate all entries in the transition probability matrix, and is therefore loose in the special case of unconstrained MDPs~\citep{agarwal2020model, li2020breaking}. In the future, we aim to extend our near-optimal sample complexity results to black-box CMDP solvers.

\section*{Acknowledgements}
We would like to thank Reza Babanezhad and Arushi Jain for helpful feedback on the paper. Csaba Szepesv\'ari gratefully acknowledges the funding from Natural Sciences and Engineering Research Council (NSERC) of Canada, ``Design.R AI-assisted CPS Design'' (DARPA)  project and the Canada CIFAR AI Chairs Program for Amii. Lin Yang is supported in part by DARPA grant HR00112190130 and NSF Award 2221871. This work was partly done while Lin Yang was visiting Deepmind.

\bibliographystyle{plainnat}
\bibliography{ref}

\newpage
\appendix
\newcommand{\appendixTitle}{%
\vbox{
    \centering
	\hrule height 4pt
	\vskip 0.2in
	{\LARGE \bf Supplementary material}
	\vskip 0.2in
	\hrule height 1pt 
}}

\appendixTitle

\section*{Organization of the Appendix}
\begin{itemize}

  \item[\ref{app:proofs-pd}] \nameref{app:proofs-pd}
    
  \item[\ref{app:proof-relaxed}] \nameref{app:proof-relaxed}
 
  \item[\ref{app:proof-strict}] \nameref{app:proof-strict}
  
  \item[\ref{app:proofs-concentration}] \nameref{app:proofs-concentration}
  
  \item[\ref{app:lb}] \nameref{app:lb}
  
  \item[\ref{app:est-zeta}] \nameref{app:est-zeta}
  
  \item[\ref{app:bai-comparison}] \nameref{app:bai-comparison}
  
\end{itemize}
\section{Proofs for primal-dual algorithm}
\label{app:proofs-pd}
\pdguarantees*
\begin{proof}
We will define the dual regret w.r.t $\lambda$ as the following quantity: 
\begin{align}
R^{d}(\lambda, T) & := \sum_{t = 0}^{T-1} (\lambdat - \lambda) \, (\consthat{\pihatt} - b')\,.
\label{eq:dual-regret-def}    
\end{align}
Using the primal update in~\cref{eq:primal-update}, for any $\pi$, 
\begin{align}
\rewardhatp{\pihatt} + \lambdat \consthat{\pihatt} \geq \rewardhatp{\pi} + \lambdat \consthat{\pi}\,.
\end{align}
Substituting $\pi = \pihatopt$, we have,
\begin{align}
\rewardhatp{\pihatopt} - \rewardhatp{\pihatt} & \leq \lambdat [\consthat{\pihatt} - \consthat{\pihatopt}]\,.
\end{align}
Since $\pihatopt$ is a solution to the empirical CMDP, $\consthat{\pihatopt} \geq b'$,
we get 
\begin{align}
\rewardhatp{\pihatopt} - \rewardhatp{\pihatt} & \leq \lambdat [\consthat{\pihatt} - b']\,. \label{eq:pd-guarantees-inter} \\
\end{align}

Starting from the definition of the dual regret (cf. \cref{eq:dual-regret-def}),
using~\cref{eq:pd-guarantees-inter} and dividing by $T$ gives
\begin{align*}
\frac{1}{T} \sum_{t = 0}^{T-1} \left[ \rewardhatp{\pihatopt} - \rewardhatp{\pihatt} \right] + \frac{\lambda}{T} \sum_{t = 0}^{T-1} (b' - \consthat{\pihatt})
& \leq \frac{R^d(\lambda, T)}{T}\,.
\end{align*}

Recall that $\bar{\pi}_T = \frac{1}{T} \sum_{t = 0}^{T-1} \pihatt$.
Then, by the definition of this ``mixture'',
we have $\frac{1}{T} \sum_{t = 0}^{T-1} \rewardhatp{\pihatt} = \rewardhatp{\bar{\pi}_T}$, and $\frac{1}{T} \sum_{t = 0}^{T-1} \consthat{\pihatt} = \consthat{\bar{\pi}_T}$. Combining this with the last inequality, we get
\begin{align}
\left[ \rewardhatp{\pihatopt} - \rewardhatp{\bar{\pi}_T} \right] + \lambda \, (b' - \consthat{\bar{\pi}_T}) & \leq R^d(\lambda, T) /T\,.
\end{align}
Note that this holds for all $\lambda$. 

Below we will show that the following inequality holds for any $\lambda\in [0,U]$:
\begin{align}
R^{d}(\lambda, T) & \leq T^{3/2} \, \frac{\epsl^2 + 2 \epsl U}{2 U (1- \gamma)} + \frac{U \sqrt{T}}{1 - \gamma}\,.
\label{eq:dual-regret-bound}        
\end{align}
This combined with the previous inequality (and the ``right'' choice of $T$, the number of updates and $\epsl$, the ``rounding parameter'') 
gives the desired bounds.
In particular, for the reward optimality gap, since $\lambda = 0 \in [0,U]$,
\begin{align*}
\rewardhatp{\pihatopt} - \rewardhatp{\pihatbar} & \leq \sqrt{T} \, \frac{\epsl^2 + 2 \epsl U}{2 U (1- \gamma)} + \frac{U}{(1 - \gamma) \sqrt{T}} < \sqrt{T} \, \frac{3 \epsl}{2 (1- \gamma)} + \frac{U}{(1 - \gamma) \sqrt{T}}\,. \tag{since $\epsl < U$} 
\end{align*}

For the constraint violation, there are two cases. The first case is when $b' - \consthat{\bar{\pi}_T}\le 0$. In this case, it also holds that $b'-\epsp-\consthat{\bar{\pi}_T}\le 0$, which is what we wanted to show.
The second case is when $b' - \consthat{\bar{\pi}_T}>0$. In this case, using the notation $[x]_{+} = \max\{x,0\}$, we have
\begin{align*}
\left[ \rewardhatp{\pihatopt} - \rewardhatp{\bar{\pi}_T} \right] + U \, \left[b' - \consthat{\bar{\pi}_T} \right]_{+} & \leq \frac{R^d(U, T)}{T} \,.
\end{align*}
Because by assumption it holds that $U > \lambda^*$, \cref{lemma:lag-constraint} is applicable and gives that
\begin{align*}
\left[b' - \consthat{\bar{\pi}_T} \right]_{+} & \leq \frac{R^d(U, T)}{T (U - \lambda^*)} \,.
\end{align*}

Hence, since $U \in [0,U]$, combining the above display with \cref{eq:dual-regret-bound} gives
\begin{align*}
\left[b' - \consthat{\bar{\pi}_T} \right] \leq \left[b' - \consthat{\bar{\pi}_T} \right]_{+} & \leq \sqrt{T} \, \frac{\epsl^2 + 2 \epsl U}{2 U (1- \gamma) \, (U - \lambda^*)} + \frac{U}{(U - \lambda^*) \, (1 - \gamma) \sqrt{T}} \\
& < \sqrt{T} \, \frac{3 \epsl}{2 (1- \gamma) \, (U - \lambda^*)} + \frac{U}{(U - \lambda^*) \, (1 - \gamma) \sqrt{T}}\,. \tag{since $\epsl < U$}
\end{align*}
Now, set $T$ such that the second term in both quantities is bounded from above by $\frac{\epsp}{2}$. 
This gives 
\begin{align*}
T & = T_0 := \frac{4 U^2}{\epsp^2 \, (1 - \gamma)^2} \left[1 + \frac{1}{(U - \lambda^*)^2} \right]\,.
\end{align*}
With $T = T_0$, the above expressions can be simplified as follows:
\begin{align*}
\rewardhatp{\pihatopt} - \rewardhatp{\pihatbar} & \leq \frac{2 U}{(1 - \gamma) \epsp} \, \left(1 + \frac{1}{U - \lambda^*} \right) \frac{3 \epsl}{2 (1- \gamma)} + \frac{\epsp}{2}\,, \\
\left[b' - \consthat{\bar{\pi}_T} \right] & \leq \frac{2 U}{(1 - \gamma) \epsp} \, \left(1 + \frac{1}{U - \lambda^*} \right) \frac{3 \epsl}{2 (1- \gamma) \, (U - \lambda^*)} + \frac{\epsp}{2}\,.
\end{align*}
Now, set $\epsl$ such that the first term in both quantities is also bounded from above by $\frac{\epsp}{2}$. For this, choose
\begin{align*}
\epsl & = \frac{\epsp^2 (1 - \gamma)^2 \, (U - \lambda^*)}{6 U}\,.
\end{align*}
With these values, the algorithm ensures that
\begin{align*}
\rewardhatp{\pihatopt} - \rewardhatp{\pihatbar}  \leq \epsp \quad \text{ and } \quad b' - \consthat{\bar{\pi}_T}  \leq \epsp.
\end{align*}

To finish the proof, it remains to show the bound \cref{eq:dual-regret-bound} on the dual regret.
For this, fix an arbitrary $\lambda \in [0,U]$.
We will track the change of $|\lambdat - \lambda|$. 
Defining $\lambdatt' := \PP_{[0,U]}[\lambdat - \eta \, (\consthat{\pihatt} - b')]$,
\begin{align*}
|\lambdatt - \lambda| & = |\mathcal{R}_{\Lambda}[\lambdatt'] - \lambda| = |\mathcal{R}_{\Lambda}[\lambdatt'] - \lambdatt' + \lambdatt' - \lambda| \leq |\mathcal{R}_{\Lambda}[\lambdatt'] - \lambdatt'| + |\lambdatt' - \lambda| \\
& \leq \epsl + |\lambdatt' - \lambda|\,.
 \tag{since $\vert \lambda - \mathcal{R}_{\Lambda}[\lambda] \vert \leq \epsl$ for all $\lambda \in [0,U]$ because of the epsilon-net.} 
\end{align*}
Squaring both sides,
\begin{align*}
|\lambdatt - \lambda|^2 & = \epsl^2 + |\lambdatt' - \lambda|^2 + 2 \epsl \, |\lambdatt' - \lambda| \leq \epsl^2 + 2 \epsl U + |\lambdatt' - \lambda|^2 \tag{since $\lambda$, $\lambdatt' \in [0,U]$, } \\
& \leq \epsl^2 + 2 \epsl U + |\lambdat - \eta \, (\consthat{\pihatt} - b') - \lambda|^2 \tag{since projections are non-expansive} \\
& = \epsl^2 + 2 \epsl U + |\lambdat - \lambda|^2 - 2 \eta \, (\lambdat - \lambda) \, (\consthat{\pihatt} - b') + \eta^2 (\consthat{\pihatt} - b')^2 \\
& \leq \epsl^2 + 2 \epsl U + |\lambdat - \lambda|^2 - 2 \eta \, (\lambdat - \lambda) \, (\consthat{\pihatt} - b') + \frac{\eta^2}{(1 - \gamma)^2}  \,,
\end{align*}
where the last inequality follows because $b'$ and the constraint value are in the $[0,1/(1-\gamma)]$ interval. Rearranging and dividing by $2 \eta$, we get
\begin{align*}
(\lambdat - \lambda) \, (\consthat{\pihatt} - b') & \leq \frac{\epsl^2 + 2 \epsl U}{2 \eta} + \frac{|\lambdat - \lambda|^2 - |\lambdatt - \lambda|^2}{2 \eta}  + \frac{\eta}{2 (1 - \gamma)^2}\,.
\end{align*}
Summing from $t = 0$ to $T-1$ and using the definition of the dual regret,
\begin{align*}
R^{d}(\lambda, T) & \leq T \, \frac{\epsl^2 + 2 \epsl U}{2 \eta} + \frac{1}{2 \eta} \sum_{t = 0}^{T-1} [|\lambdat - \lambda|^2 - |\lambdatt - \lambda|^2] + \frac{\eta T}{2 (1 - \gamma)^2} \,.
\end{align*}
Telescoping, bounding $|\lambda_0 - \lambda|$ by $U$ and dropping a negative term gives
\begin{align*}
R^{d}(\lambda, T)& \leq T \, \frac{\epsl^2 + 2 \epsl U}{2 \eta} + \frac{U^2}{2 \eta} + \frac{\eta T}{2 (1 - \gamma)^2} \,.
\end{align*}
Setting $\eta = \frac{U (1 - \gamma)}{\sqrt{T}}$,
\begin{align}
R^{d}(\lambda, T) & \leq T^{3/2} \, \frac{\epsl^2 + 2 \epsl U}{2 U (1- \gamma)} + \frac{U \sqrt{T}}{1 - \gamma}\,,
\label{eq:dual-regret-bound2}        
\end{align}
which finishes the proof.

\end{proof}

\begin{thmbox}
\begin{lemma}[Bounding the dual variable]
The objective~\cref{eq:emp-CMDP} satisfies strong duality. Defining $\pi^*_c := \argmax \const{\pi}$. We consider two cases: (1) If $b' = b - \epsilon'$ for $\epsilon' > 0$ and event $\mathcal{E}_1 = \left\{\abs{\consthat{\pi^*_c} -  \const{\pi^*_c}}  \leq \frac{\epsilon'}{2} \right\}$ holds, then $\lambda^* \leq \frac{2 (1 + \omega)}{\epsilon' (1 - \gamma)}$ and (2) If $b' = b + \Delta$ for $\Delta \in \left(0,\frac{\zeta}{2}\right)$ and event $\mathcal{E}_2  = \left\{\abs{\consthat{\pi^*_c} -  \const{\pi^*_c}} \leq \frac{\zeta}{2} - \Delta 
\right\}$ holds, then $\lambda^* \leq \frac{2 (1 + \omega)}{\zeta (1 - \gamma)}$. 
\label{lemma:dual-bound} 
\end{lemma}
\end{thmbox}
\begin{proof}
Writing the empirical CMDP in~\cref{eq:emp-CMDP} in its Lagrangian form,
\begin{align*}
\rewardhatp{\pihatopt} & = \max_{\pi} \min_{\lambda \geq 0}  \rewardhatp{\pi} + \lambda [\consthat{\pi} - b'] 
\intertext{Using the linear programming formulation of CMDPs in terms of the state-occupancy measures $\mu$, we know that both the objective and the constraint are linear functions of $\mu$, and strong duality holds w.r.t $\mu$. Since $\mu$ and $\pi$ have a one-one mapping, we can switch the min and the max~\citep{paternain2019constrained}, implying,} & = \min_{\lambda \geq 0} \max_{\pi} \rewardhatp{\pi} + \lambda [\consthat{\pi} - b']
\intertext{Since $\lambda^*$ is the optimal dual variable for the empirical CMDP in~\cref{eq:emp-CMDP}, }
& = \max_{\pi} \rewardhatp{\pi} + \lambda^* \, [\consthat{\pi} - b'] \\
\intertext{Define $\pi^*_c := \argmax \const{\pi}$ and $\pihat^*_c := \argmax \consthat{\pi}$}
& \geq \rewardhatp{\pihat^*_c} + \lambda^* \, [\consthat{\pihat^*_c} - b'] \\
& = \rewardhatp{\pihat^*_c} + \lambda^* \, \left[\left(\consthat{\pihat^*_c} - \const{\pi^*_c} \right) + (\const{\pi^*_c} - b) + (b - b') \right] \\
\intertext{By definition, $\zeta = \const{\pi^*_c} - b$}
& = \rewardhatp{\pihat^*_c} + \lambda^* \, \left[\left(\consthat{\pihat^*_c} - \consthat{\pi^*_c} \right) + \left(\consthat{\pi^*_c} -  \const{\pi^*_c} \right) + \zeta + (b - b') \right] 
\intertext{By definition of $\pihat^*_c$, $\left(\consthat{\pihat^*_c} - \consthat{\pi^*_c} \right) \geq 0$}
\rewardhatp{\pihatopt} & \geq \rewardhatp{\pihat^*_c} + \lambda^* \, \left[\zeta + (b - b') - \abs{\consthat{\pi^*_c} -  \const{\pi^*_c}}  \right] 
\end{align*}

1) If $b' = b - \epsilon'$ for $\epsilon' > 0$. Hence,   
\begin{align*}
\rewardhatp{\pihatopt} & \geq \rewardhatp{\pihat^*_c} + \lambda^* \, \left[\zeta + \epsilon' - \abs{\consthat{\pi^*_c} -  \const{\pi^*_c}}  \right] \\
\intertext{If the event $\mathcal{E}_1$ holds, $\abs{\consthat{\pi^*_c} -  \const{\pi^*_c}} \leq \frac{\epsilon'}{2}$, implying, $\abs{\consthat{\pi^*_c} -  \const{\pi^*_c}} < \zeta + \frac{\epsilon'}{2}$, then,}
& \geq \rewardhatp{\pihat^*_c} + \lambda^* \, \frac{\epsilon'}{2} \\
\implies \lambda^* & \leq \frac{2}{\epsilon'} [\rewardhatp{\pihatopt} - \rewardhatp{\pihat^*_c}] \leq \frac{2 (1 + \omega)}{\epsilon' (1 - \gamma)}
\end{align*}

2) If $b' = b + \Delta$ for $\Delta \in \left(0,\frac{\zeta}{2}\right)$. Hence, 
\begin{align*}
\rewardhatp{\pihatopt} & \geq  \rewardhatp{\pihat^*_c} + \lambda^* \, \left[\zeta - \Delta - \abs{\consthat{\pi^*_c} -  \const{\pi^*_c}}  \right]  
\intertext{If the event $\mathcal{E}_2$ holds, $\abs{\consthat{\pi^*_c} -  \const{\pi^*_c}} \leq \frac{\zeta}{2} - \Delta$ for $\Delta < \frac{\zeta}{2}$, then,}
& \geq  \rewardhatp{\pihat^*_c} + \lambda^* \, \frac{\zeta}{2} \\
\implies \lambda^* & \leq \frac{2}{\zeta} [\rewardhatp{\pihatopt} - \rewardhatp{\pihat^*_c}] \leq \frac{2 (1 + \omega)}{\zeta (1 - \gamma)}
\end{align*}
\end{proof}

\begin{thmbox}
\begin{lemma}[Lemma B.2 of~\citet{jain2022towards}]
For any $C > \lambda^*$ and any $\tilde{\pi}$ s.t. $\rewardhatp{\pihatopt} - \rewardhatp{\tilde{\pi}} + C [b - \consthat{\tilde{\pi}}]_{+} \leq \beta$, we have $[b - \consthat{\tilde{\pi}}]_{+} \leq \frac{\beta}{C - \lambda^*}$.
\label{lemma:lag-constraint}
\end{lemma}
\end{thmbox}
\begin{proof}
Define $\nu(\tau) = \max_{\pi} \{\reward{\pi} \mid \const{\pi} \geq b + \tau \}$ and note that by definition, $\nu(0) = \reward{\pihatopt}$ and that $\nu$ is a decreasing function for its argument.

Let $\lag{\pi}{\lambda} = \reward{\pi}+\lambda(\const{\pi}-b)$. Then,
for any policy $\pi$ s.t. $\const{\pi} \geq b + \tau$, we have
\begin{align}
\lag{\pi}{\lambda^*} & \leq \max_{\pi'} \lag{\pi'}{\lambda^*} \nonumber \\
&= \reward{\pihatopt} &\tag{by strong duality} \\ 
& = \nu(0) & \tag{from above relation} \\
\implies \nu(0) - \tau \lambda^* & \geq \lag{\pi}{\lambda^*} - \tau \lambda^* = \reward{\pi} + \lambda^* \underbrace{(\const{\pi} - b - \tau)}_{\text{Non-negative}} \nonumber \\
\implies \nu(0) - \tau \lambda^* & \geq \max_{\pi} \{\reward{\pi} \mid \const{\pi} \geq b + \tau \} = \nu(\tau) \,.\nonumber \\
\implies \tau \lambda^* \leq \nu(0) - \nu(\tau)\,. \label{eq:inter-1}
\end{align}

Now we choose $\tautil = -(b - \const{\pitil})_{+}$.
\begin{align*}
(C - \lambda^*) |\tautil| &= \lambda^* \tautil + C |\tautil| & \tag{since $\tautil \leq 0$} \\
& \leq \nu(0) - \nu(\tautil) + C |\tautil| & \tag{\cref{eq:inter-1}} \\
& = \reward{\pihatopt} - \reward{\pitil} + C |\tautil| + \reward{\pitil} - \nu(\tautil) & \tag{definition of $\nu(0)$} \\
& = \reward{\pihatopt} - \reward{\pitil} + C (b - \const{\pitil})_{+} + \reward{\pitil} - \nu(\tautil) \\
& \leq \beta + \reward{\pitil} - \nu(\tautil)\,.
\intertext{Now let us bound $\nu(\tautil)$:} 
\nu(\tautil) & = \max_{\pi} \{\reward{\pi} \mid \const{\pi} \geq b - (b - \const{\pitil})_{+} \}  \\
& \geq \max_{\pi} \{\reward{\pi} \mid \const{\pi} \geq \const{\pitil} \} & \tag{tightening the constraint} \\
\nu(\tautil) & \geq \reward{\pitil} 
\implies (C - \lambda^*) |\tautil| 
 \leq \beta \implies (b - \const{\pitil})_{+} \leq \frac{\beta}{C - \lambda^*} 
\end{align*}
\end{proof}
\section{Proof of~\cref{thm:ub-relaxed}}
\label{app:proof-relaxed}
\ubrelaxed*
\begin{proof}
We fill in the details required for the proof sketch in the main paper. Proceeding according to the proof sketch, we first detail the computation of $T$ and $\epsl$ for the primal-dual algorithm. Recall that $U = \frac{32}{5 \epsilon \, (1 - \gamma)}$ and $\epsp = \frac{\epsilon}{4}$. Using~\cref{thm:pd-guarantees}, we need to set 
\begin{align*}
T & = \frac{4 U^2}{\epsp^2 \, (1 - \gamma)^2} \left[1 + \frac{1}{(U - \lambda^*)^2} \right]  = \frac{64}{\epsilon^2 (1 - \gamma)^2} \left[1 + \frac{1}{(U - \lambda^*)^2} \right]
\intertext{Recall that $|\lambda^*| \leq C := \frac{16}{5 \epsilon \, (1 - \gamma)}$ and $U = 2 C$. Simplifying,}
& \leq \frac{256}{\epsilon^2 (1 - \gamma)^2} \left[C^2 + 1 \right] < \frac{512}{\epsilon^2 (1 - \gamma)^2} C^2 = \frac{512}{\epsilon^2 (1 - \gamma)^2} \, \frac{256}{25 \epsilon^2 \, (1 - \gamma)^2} \\
\implies T &= O \left(\nicefrac{1}{\epsilon^4 (1 - \gamma)^4}\right).
\end{align*}
Using~\cref{thm:pd-guarantees}, we need to set $\epsl$, 
\begin{align*}
\epsl & = \frac{\epsp^2 (1 - \gamma)^2 \, (U - \lambda^*)}{6 U} = \frac{\epsilon^2 (1 - \gamma)^2 \, (U - \lambda^*)}{96 U} \leq \frac{\epsilon^2 (1 - \gamma)^2}{96} \\
\implies \epsl &= O \left(\epsilon^2 (1 - \gamma)^2\right).
\end{align*}

For bounding the concentration terms for $\pihatbar$ in~\cref{eq:conc-relaxed}, we use~\cref{thm:main-concentration-emp} with $U = \frac{32}{5 \epsilon \, (1 - \gamma)}$, $\omega = \frac{\epsilon (1 - \gamma)}{8}$ and $\epsl = \frac{\epsilon^2 (1 - \gamma)^2}{96}$. In this case, $\iota = \frac{\omega \, \delta \, (1-\gamma) \, \epsl}{30 \, U |S||A|^2} = O\left(\frac{\delta \epsilon^4 \, (1 - \gamma)^4}{SA^2}\right)$ and
\begin{align*}
C(\delta) = 72 \log \left(\frac{16 (1 + U + \omega) \, S A \log\left(\nicefrac{e}{1 - \gamma}\right)}{(1 - \gamma)^2 \, \iota \, \delta} \right) = O\left(\log\left(\frac{S^2 A^3}{\delta^2 \epsilon^5 (1 - \gamma)^7}\right)\right).
\end{align*}
With this value of $C(\delta)$, in order to satisfy the concentration bounds for $\pihatbar$, we require that 
\begin{align*}
2 \sqrt{\frac{C(\delta)}{N \cdot (1-\gamma)^3 }} \leq \frac{\epsilon}{4} \implies N \geq O\left(\frac{C(\delta)}{(1 - \gamma)^3 \, \epsilon^2}\right)     
\end{align*}
We use the~\cref{lemma:main-concentration-opt} to bound the remaining concentration terms for $\piopt$ and $\pi^*_c$ in~\cref{eq:conc-relaxed}. In this case, for $C'(\delta) = 72 \log \left(\frac{4 S \log(e/1-\gamma)}{\delta}\right)$, we require that,
\begin{align*}
2 \sqrt{\frac{C'(\delta)}{N \cdot (1-\gamma)^3 }} \leq \frac{\epsilon}{4} \implies N \geq O\left(\frac{C'(\delta)}{(1 - \gamma)^3 \, \epsilon^2}\right)         
\end{align*}
Hence, if $N \geq \tilde{O}\left(\frac{\log(1/\delta)}{(1 - \gamma)^3 \, \epsilon^2}\right)$, the bounds in~\cref{eq:conc-relaxed} are satisfied, completing the proof.
\end{proof}

\begin{thmbox}
\begin{lemma}[Decomposing the suboptimality]
For $b' = b - \frac{\epsilon - \epsp}{2}$, if (i) $\epsp < \epsilon$, and (ii) the following conditions are satisfied, 

\begin{align*}
\abs{\const{\pihatbar} - \consthat{\pihatbar}} \leq \frac{\epsilon - \epsp}{2} \, \text{;} \, \abs{\const{\piopt} - \consthat{\piopt}} \leq \frac{\epsilon - \epsp}{2} 
\end{align*}
where $\pi^*_c := \argmax \const{\pi}$, then (a) policy $\pihatbar$ violates the constraint by at most $\epsilon$ i.e. $\const{\pihatbar} \geq b - \epsilon$ and (b) its optimality gap can be bounded as:
\begin{align*}
\reward{\piopt} - \reward{\pihatbar} & \leq \frac{2 \omega}{1 - \gamma} + \epsp + \abs{\rewardp{\piopt} - \rewardhatp{\piopt}} + \abs{\rewardhatp{\pihatbar} - \rewardp{\pihatbar}}
\end{align*}
\label{lemma:decomposition-relaxed}
\end{lemma}
\end{thmbox}
\begin{proof}
From~\cref{thm:pd-guarantees}, we know that, 
\begin{align*}
& \consthat{\pihatbar} \geq b' - \epsp 
\implies \const{\pihatbar} \geq \const{\pihatbar} - \consthat{\pihatbar} + b' - \epsp \geq - \abs{\const{\pihatbar} - \consthat{\pihatbar}} + b' - \epsp \\
\end{align*}
Since we require $\pihatbar$ to violate the constraint in the true CMDP by at most $\epsilon$, we require $\const{\pihatbar} \geq b - \epsilon$. From the above equation, a sufficient condition for ensuring this is,
\begin{align*}
 - \abs{\const{\pihatbar} - \consthat{\pihatbar}} + b' - \epsp \geq b - \epsilon \,,
 \end{align*}
meaning that we require
\begin{align*}
 \abs{\const{\pihatbar} - \consthat{\pihatbar}} \leq (b' - b) - \epsp + \epsilon.
\end{align*}
Plugging in the value of $b'$, we see that this sufficient condition indeed holds, by our assumption that $\abs{\const{\pihatbar} - \consthat{\pihatbar}} \leq \frac{\epsilon - \epsp}{2}$.

Let $\piopt$ be the solution to~\cref{eq:true-CMDP}. Our next goal is to show that $\piopt$ is feasible for the constrained problem in~\cref{eq:emp-CMDP}, i.e., $\consthat{\piopt} \geq b'$. We have
\begin{align*}
& \const{\piopt} \geq b \implies \consthat{\piopt} \geq b - \abs{\const{\piopt} - \consthat{\piopt}}  
\intertext{Since we require $\consthat{\piopt} \geq b'$, using the above equation, a sufficient condition to ensure this is}
& b - \abs{\const{\piopt} - \consthat{\piopt}}  \geq b' \text{meaning that we require} \abs{\const{\piopt} - \consthat{\piopt}} \leq b - b'.
\intertext{Since $b' = b - \frac{\epsilon - \epsp}{2}$, we require that}
&\abs{\const{\piopt} - \consthat{\piopt}} \leq \frac{\epsilon - \epsp}{2}. 
\end{align*}
Given that the above statements hold, we can decompose the suboptimality in the reward value function as follows:
\begin{align*}
& \reward{\piopt} - \reward{\pihatbar} \\ & = \reward{\piopt} - \rewardp{\piopt} + \rewardp{\piopt} - \reward{\pihatbar} \\
&= [\reward{\piopt} - \rewardp{\piopt}] + \rewardp{\piopt} - \rewardhatp{\piopt} + \rewardhatp{\piopt}  - \reward{\pihatbar} \\
& \leq [\reward{\piopt} - \rewardp{\piopt}] + [\rewardp{\piopt} - \rewardhatp{\piopt}] + \rewardhatp{\pihatopt}  - \reward{\pihatbar} & \tag{By optimality of $\pihatopt$ and since we have ensured that $\piopt$ is feasible for~\cref{eq:emp-CMDP}} \\
& = [\reward{\piopt} - \rewardp{\piopt}] + [\rewardp{\piopt} - \rewardhatp{\piopt}] + [\rewardhatp{\pihatopt} - \rewardhatp{\pihatbar}] + \rewardhatp{\pihatbar} - \reward{\pihatbar} \\
& = \underbrace{[\reward{\piopt} - \rewardp{\piopt}]}_{\text{Perturbation Error}} + \underbrace{[\rewardp{\piopt} - \rewardhatp{\piopt}]}_{\text{Concentration Error}} + \underbrace{[\rewardhatp{\pihatopt} - \rewardhatp{\pihatbar}]}_{\text{Primal-Dual Error}} + \underbrace{[\rewardhatp{\pihatbar} - \rewardp{\pihatbar}]}_{\text{Concentration Error}}
+ \underbrace{[\rewardp{\pihatbar} - \reward{\pihatbar}]}_{\text{Perturbation Error}} \\
\intertext{For a perturbation magnitude equal to $\omega$, we use~\cref{lemma:perturbation-error} to  bound both perturbation errors by $\frac{\omega}{1 - \gamma}$. Using~\cref{thm:pd-guarantees} to bound the primal-dual error by $\epsp$,}
& \reward{\piopt} - \reward{\pihatbar} \leq \frac{2 \omega}{1 - \gamma} + \epsp +  \underbrace{[\rewardp{\piopt} - \rewardhatp{\piopt}]}_{\text{Concentration Error}} + \underbrace{[\rewardhatp{\pihatbar} - \rewardp{\pihatbar}]}_{\text{Concentration Error}}.
\end{align*}
\end{proof}

\section{Proof of~\cref{thm:ub-strict}}
\label{app:proof-strict}
\ubstrict*
\begin{proof}
We fill in the details required for the proof sketch in the main paper. Proceeding according to the proof sketch, we first detail the computation of $T$ and $\epsl$ for the primal-dual algorithm. Recall that $U = \frac{8}{\zeta (1 - \gamma)}$, $\Delta = \frac{\epsilon (1- \gamma) \zeta}{40}$ and $\epsp = \frac{\Delta}{5}$. Using~\cref{thm:pd-guarantees}, we need to set 
\begin{align*}
T & = \frac{4 U^2}{\epsp^2 \, (1 - \gamma)^2} \left[1 + \frac{1}{(U - \lambda^*)^2} \right]  = \frac{100}{\Delta^2 (1 - \gamma)^2} \left[1 + \frac{1}{(U - \lambda^*)^2} \right]
\intertext{Recall that $|\lambda^*| \leq C :=  \frac{4}{\zeta (1 - \gamma)}$ and $U = 2 C$. Simplifying,}
& \leq \frac{400}{\Delta^2 (1 - \gamma)^2} \left[C^2 + 1 \right] < \frac{800}{\Delta^2 (1 - \gamma)^2} C^2 = \frac{800}{\Delta^2 (1 - \gamma)^2} \, \frac{16}{\zeta^2 \, (1 - \gamma)^2} \\
\implies T &\leq \frac{800 \cdot 1600}{\epsilon^2 \zeta^2 (1 - \gamma)^4} \, \frac{16}{\zeta^2 \, (1 - \gamma)^2} = O \left(\nicefrac{1}{\epsilon^2 \, \zeta^4 \, (1 - \gamma)^6}\right).
\end{align*}
Using~\cref{thm:pd-guarantees}, we need to set $\epsl$, 
\begin{align*}
\epsl & = \frac{\epsp^2 (1 - \gamma)^2 \, (U - \lambda^*)}{6 U} = \frac{\Delta^2 (1 - \gamma)^2 \, (U - \lambda^*)}{150 U} \leq \frac{\Delta^2 (1 - \gamma)^2}{150} \\
\implies \epsl &\leq \frac{\epsilon^2 \, \zeta^2 \, (1 - \gamma)^4}{150 \cdot 1600} = O \left(\epsilon^2 \, \zeta^2 \, (1 - \gamma)^4 \right).
\end{align*}
For bounding the concentration terms for $\pihatbar$ in~\cref{eq:conc-strict}, we use~\cref{thm:main-concentration-emp} with $U = \frac{8}{\zeta (1 - \gamma)}$, $\omega = \frac{\epsilon (1 - \gamma)}{10}$ and $\epsl = \frac{\epsilon^2 \, \zeta^2 \, (1 - \gamma)^4}{150 \cdot 1600}$. In this case, $\iota = \frac{\omega \, \delta \, (1-\gamma) \, \epsl}{30 \, U |S||A|^2} = O\left(\frac{\delta \epsilon^3 \zeta^3 (1 - \gamma)^7}{SA^2}\right)$ and
\begin{align*}
C(\delta) = 72 \log \left(\frac{16 (1 + U + \omega) \, S A \log\left(\nicefrac{e}{1 - \gamma}\right)}{(1 - \gamma)^2 \, \iota \, \delta} \right) = O\left(\log\left(\frac{S^2 A^3}{(1 - \gamma)^{10} \delta^2 \epsilon^3 \zeta^4}\right)\right).
\end{align*}
With this value of $C(\delta)$, in order to satisfy the concentration bounds for $\pihatbar$, we require that 
\begin{align*}
2 \sqrt{\frac{C(\delta)}{N \cdot (1-\gamma)^3 }} \leq \frac{\Delta}{5} \implies N \geq O\left(\frac{C(\delta)}{(1 - \gamma)^3 \, \Delta^2}\right) \geq O\left(\frac{C(\delta)}{(1 - \gamma)^5 \, \zeta^2 \, \epsilon^2}\right)    
\end{align*}
We use the~\cref{lemma:main-concentration-opt} to bound the remaining concentration terms for $\piopt$ and $\pi^*_c$ in~\cref{eq:conc-strict}. In this case, for $C'(\delta) = 72 \log \left(\frac{4 S \log(e/1-\gamma)}{\delta}\right)$, we require that,
\begin{align*}
2 \sqrt{\frac{C'(\delta)}{N \cdot (1-\gamma)^3 }} \leq \frac{\Delta}{5} \implies N \geq O\left(\frac{C'(\delta)}{(1 - \gamma)^3 \, \Delta^2}\right) \geq O\left(\frac{C'(\delta)}{(1 - \gamma)^5 \, \zeta^2 \, \epsilon^2}\right)            
\end{align*}
Hence, if $N \geq \tilde{O}\left(\frac{\log(1/\delta)}{(1 - \gamma)^5 \, \zeta^2 \, \epsilon^2}\right)$, the bounds in~\cref{eq:conc-strict} are satisfied, completing the proof.
\end{proof}

\clearpage
\begin{thmbox}
\begin{lemma}[Decomposing the suboptimality]
For a fixed $\Delta > 0$ and $\epsp < \Delta$, if $b' = b + \Delta$, then the following conditions are satisfied, 
\begin{align*}
\abs{\const{\pihatbar} - \consthat{\pihatbar}} \leq \Delta - \epsp \, \text{;} \, \abs{\const{\piopt} - \consthat{\piopt}} \leq \Delta 
\end{align*}
then (a) policy $\pihatbar$ satisfies the constraint i.e. $\const{\pihatbar} \geq b$ and (b) its optimality gap can be bounded as:
\begin{align*}
\reward{\piopt} - \reward{\pihatbar} & \leq \frac{2 \omega}{1 - \gamma} + \epsp + 2 \Delta \lambda^* + \abs{\rewardp{\piopt} - \rewardhatp{\piopt}} +  \abs{\rewardhatp{\pihatbar} - \rewardp{\pihatbar}}.
\end{align*}
\label{lemma:decomposition-strict}
\end{lemma}
\end{thmbox}
\begin{proof}
Compared to~\cref{eq:emp-CMDP}, we define a slightly modified CMDP problem by changing the constraint RHS to $b''$ for some $b''$ to be specified later. We denote its corresponding optimal policy as $\pitildeopt$. In particular, 
\begin{align}
\pitildeopt \in \argmax_{\pi} \rewardhatp{\pi} \, \text{s.t.} \, \consthat{\pi} \geq b''
\label{eq:emp-CMDP-2}
\end{align}
From~\cref{thm:pd-guarantees}, we know that, 
\begin{align*}
& \consthat{\pihatbar} \geq b' - \epsp 
\implies \const{\pihatbar} \geq \const{\pihatbar} - \consthat{\pihatbar} + b' - \epsp \geq - \abs{\const{\pihatbar} - \consthat{\pihatbar}} + b' - \epsp \\
\intertext{Since we require $\pihatbar$ to satisfy the constraint in the true CMDP, we require $\const{\pihatbar} \geq b$. From the above equation, a sufficient condition for ensuring this is,}
& - \abs{\const{\pihatbar} - \consthat{\pihatbar}} + b' - \epsp \geq b \\
& \text{meaning that we require} \abs{\const{\pihatbar} - \consthat{\pihatbar}} \leq (b' - b) - \epsp.
\end{align*}
In the subsequent analysis, we will require $\piopt$ to be feasible for the constrained problem in~\cref{eq:emp-CMDP-2}. This implies that we require $\consthat{\piopt} \geq b''$. Since $\piopt$ is the solution to~\cref{eq:true-CMDP}, we know that, 
\begin{align*}
& \const{\piopt} \geq b \implies \consthat{\piopt} \geq b - \abs{\const{\piopt} - \consthat{\piopt}}  
\intertext{Since we require $\consthat{\piopt} \geq b''$, using the above equation, a sufficient condition to ensure this is}
& b - \abs{\const{\piopt} - \consthat{\piopt}}  \geq b'' \text{meaning that we require} \abs{\const{\piopt} - \consthat{\piopt}} \leq b - b''.
\intertext{Hence we require the following statements to hold:}
& \abs{\const{\pihatbar} - \consthat{\pihatbar}} \leq (b' - b) - \epsp \quad \text{;} \quad \abs{\const{\piopt} - \consthat{\piopt}} \leq b - b''.
\end{align*}
Given that the above statements hold, we can decompose the suboptimality in the reward value function as follows:
\begin{align*}
\reward{\piopt} - \reward{\pihatbar} & = \reward{\piopt} - \rewardp{\piopt} + \rewardp{\piopt} - \reward{\pihatbar} \\
& = [\reward{\piopt} - \rewardp{\piopt}] + [\rewardp{\piopt} - \rewardhatp{\piopt}] + \rewardhatp{\piopt} - \reward{\pihatbar} \\
& \leq [\reward{\piopt} - \rewardp{\piopt}] + [\rewardp{\piopt} - \rewardhatp{\piopt}] + \rewardhatp{\pitildeopt} - \reward{\pihatbar}  & \tag{By optimality of $\pihatopt$ and since we have ensured that $\piopt$ is feasible for~\cref{eq:emp-CMDP-2}} \\
& = [\reward{\piopt} - \rewardp{\piopt}] + [\rewardp{\piopt} - \rewardhatp{\piopt}] + [\rewardhatp{\pitildeopt} - \rewardhatp{\pihatopt}] + \rewardhatp{\pihatopt} - \reward{\pihatbar} \\
& = [\reward{\piopt} - \rewardp{\piopt}] + [\rewardp{\piopt} - \rewardhatp{\piopt}] + [\rewardhatp{\pitildeopt} - \rewardhatp{\pihatopt}] + \rewardhatp{\pihatopt} - \reward{\pihatbar} \\
& = [\reward{\piopt} - \rewardp{\piopt}] + [\rewardp{\piopt} - \rewardhatp{\piopt}] + [\rewardhatp{\pitildeopt} - \rewardhatp{\pihatopt}] + [\rewardhatp{\pihatopt} - \rewardhatp{\pihatbar}] \\ & + \rewardhatp{\pihatbar} - \reward{\pihatbar} \\
& = \underbrace{[\reward{\piopt} - \rewardp{\piopt}]}_{\text{Perturbation Error}} + \underbrace{[\rewardp{\piopt} - \rewardhatp{\piopt}]}_{\text{Concentration Error}} + \underbrace{[\rewardhatp{\pitildeopt} - \rewardhatp{\pihatopt}]}_{\text{Sensitivity Error}} + \underbrace{[\rewardhatp{\pihatopt} - \rewardhatp{\pihatbar}]}_{\text{Primal-Dual Error}} \\ & + \underbrace{[\rewardhatp{\pihatbar} - \rewardp{\pihatbar}]}_{\text{Concentration Error}} + \underbrace{[\rewardp{\pihatbar} - \reward{\pihatbar}]}_{\text{Perturbation Error}} \\
\intertext{For a perturbation magnitude equal to $\omega$, we use~\cref{lemma:perturbation-error} to  bound both perturbation errors by $\frac{\omega}{1 - \gamma}$. Using~\cref{thm:pd-guarantees} to bound the primal-dual error by $\epsp$,}
& \leq \frac{2 \omega}{1 - \gamma} + \epsp + \underbrace{[\rewardp{\piopt} - \rewardhatp{\piopt}]}_{\text{Concentration Error}} + \underbrace{[\rewardhatp{\pitildeopt} - \rewardhatp{\pihatopt}]}_{\text{Sensitivity Error}} + \underbrace{[\rewardhatp{\pihatbar} - \rewardp{\pihatbar}]}_{\text{Concentration Error}} 
\end{align*}
Since $b' = b + \Delta$ and setting $b'' = b - \Delta$, we use~\cref{lemma:sensitivity} to bound the sensitivity error term,  
\begin{align*}
\reward{\piopt} - \reward{\pihatbar} & \leq \frac{2 \omega}{1 - \gamma} + \epsp + 2 \Delta \lambda^* +  \underbrace{[\rewardp{\piopt} - \rewardhatp{\piopt}]}_{\text{Concentration Error}} + \underbrace{[\rewardhatp{\pihatbar} - \rewardp{\pihatbar}]}_{\text{Concentration Error}} 
\end{align*}
With these values of $b'$ and $b''$, we require the following statements to hold, 
\begin{align*}
\abs{\const{\pihatbar} - \consthat{\pihatbar}} & \leq \Delta - \epsp \quad \text{;} \quad \abs{\const{\piopt} - \consthat{\piopt}} \leq \Delta.
\end{align*}
\end{proof}

\begin{thmbox}
\begin{lemma}[Bounding the sensitivity error]
If $b' = b + \Delta$ and $b'' = b - \Delta$ in~\cref{eq:emp-CMDP} and~\cref{eq:emp-CMDP-2} such that,  
\[
\pihatopt \in \argmax_{\pi} \rewardhatp{\pi} \, \text{s.t.} \, \consthat{\pi} \geq b + \Delta
\]
\[
\pitildeopt \in \argmax_{\pi} \rewardhatp{\pi} \, \text{s.t.} \, \consthat{\pi} \geq b - \Delta \,,
\] 
then the sensitivity error term can be bounded by:
\begin{align*}
\abs{\rewardhatp{\pihatopt} - \rewardhatp{\pitildeopt}} & \leq 2 \Delta \lambda^*.
\end{align*}
\label{lemma:sensitivity}
\end{lemma}
\end{thmbox}
\begin{proof}
Writing the empirical CMDP in~\cref{eq:emp-CMDP} in its Lagrangian form,
\begin{align*}
\rewardhatp{\pihatopt} & = \max_{\pi} \min_{\lambda \geq 0}  \rewardhatp{\pi} + \lambda [\consthat{\pi} - (b + \Delta)]  \\
& = \min_{\lambda \geq 0} \max_{\pi} \rewardhatp{\pi} + \lambda [\consthat{\pi} - (b + \Delta)] & \tag{By strong duality~\cref{lemma:dual-bound}} \\
\intertext{Since $\lambda^*$ is the optimal dual variable for the empirical CMDP in~\cref{eq:emp-CMDP}, }
& = \max_{\pi} \rewardhatp{\pi} + \lambda^* \, [\consthat{\pi} - (b + \Delta)] \\
& \geq \rewardhatp{\pitildeopt} + \lambda^* \, [\consthat{\pitildeopt} - (b + \Delta)] & \tag{The relation holds for $\pi = \pitildeopt$.} \\
\intertext{Since $\consthat{\pitildeopt} \geq b - \Delta$,} \rewardhatp{\pihatopt} & \geq \rewardhatp{\pitildeopt} - 2 \lambda^* \Delta \\
\implies \rewardhatp{\pitildeopt} - \rewardhatp{\pihatopt} & \leq 2 \Delta \lambda^*
\intertext{Since the CMDP in~\cref{eq:emp-CMDP-2} (with $b'' = b - \Delta$) is a less constrained problem than the one in~\cref{eq:emp-CMDP} (with $b' = b + \Delta$), $\rewardhatp{\pitildeopt} \geq \rewardhatp{\pihatopt}$, and hence,}
\abs{\rewardhatp{\pitildeopt} - \rewardhatp{\pihatopt}} & \leq 2 \Delta \lambda^*.
\end{align*}
\end{proof}
\clearpage
\section{Concentration proofs}
\label{app:proofs-concentration}
\begin{thmbox}
\main*
\end{thmbox}
\begin{proof}
Since the policy $\pihatopta$ depends on the sampling, we can not directly apply the standard concentration results to bound $\norminf{\rewardbhat{\pihatopta}  - \rewardbhat{\pihatopta}}$. We thus seek to apply a critical lemma established in~\citet{li2020breaking}. It begins with introducing a sequence of vectors for a general data-dependent policy $\pi$ and reward $\beta$, defined recursively as 
\begin{align*}
    V_{\beta}^{\pi,(0)} = (I-\gamma P^{\pi})^{-1}\beta^{\pi} \quad\text{and} \quad V_{\beta}^{\pi, (l)} = (I-\gamma P^{\pi})^{-1}\sqrt{\mathrm{Var}_{P_{\pi}}(V_{\beta}^{\pi, (l-1)})}, ~\forall l\ge 1.
\end{align*}
In their Lemma~2 (restated below), they show that if certain concentration relation can be established between the empirical and ground truth MDP, then $\norminf{\rewardbhat{\pi}  - \rewardbhat{\pi}}$ can be bounded. 
\begin{lemma}[Lemma~2 of \cite{li2020breaking}]
\label{lemma:bern2con}
For a data-dependent policy $\pi$, suppose there exists a $\nu_1\ge 0$ such that $\{V_{\beta}^{\pi, (l)}\}$ obeys
\begin{align*}
    \left|
    (\hat{\cP}_{\pi} - \cP_{\pi})V^{\pi,(l)}_{\beta} 
    \right|\le \sqrt{\frac{\nu_1}{N}}\sqrt{\mathrm{Var}_{\cP_{\pi}}[V_{\beta}^{{\pi},(l)}]} +\frac{\nu_1\norminf{V_{\beta}^{\pi, (l)}}}{N}, \text{ for all } 0\le l \le \log\left(
    \frac{e}{1-\gamma}\right).
\end{align*}
Suppose that $N\ge \frac{16e^2}{1-\gamma}\nu_1$. Then
\[
\norminf{\hat{V}_{\beta}^{\pi} - V_{\beta}^{\pi}} \le \frac{6}{1-\gamma}\cdot \sqrt{\frac{\nu_1}{N(1-\gamma)}} \norminf{\beta}.
\]
\end{lemma}
To use this lemma for $\pi = \pihatopta$, we will need to establish Bernstein-type bounds on $(\cP_{s,a}-\hat{\cP}_{s,a}) V_{\beta}^{\pihatopta, (l)}$ for all $(s,a)$ and integer $0\le l \le \log(e/(1-\gamma))$. Since $\pihatopta$ depends on the sampling, a direct concentration bound is not possible. Instead, we will first bound $(\cP_{s,a}-\hat{\cP}_{s,a})V_{\beta}^{\pi, (l)}$ for all $\pi\in \Pi_{s,a}$, where $\Pi_{s,a}$ is a random set independent of $\hat{\cP}$, and then show that $\hat{\pi}^*\in \Pi_{s,a}$ with good probability. 
 
First, we describe the construction of $\Pi_{s,a}$. We will follow the ideas in~\citet{agarwal2020model} and \citet{li2020breaking}, and construct an absorbing empirical MDP $\hat{M}_{s,a}$, which is the same as the original empirical MDP, but state-action pair $(s,a)$ is absorbing, i.e., $\hat{\cP}'(s^\prime | s, a) = 1$ if and only if $s^\prime = s$. The reward for $(s,a)$ is equal to $u$. We define $\hat{V}_{s,a,\alpha, u}^{\pi}$ and $\hat{Q}_{s,a,\alpha, u}^{\pi}$ to be the value function and $Q$-function of policy $\pi$ for $\hat{M}_{s,a}$ with reward function $\alpha$, and define $\hat{\pi}_{s,a,\alpha, u}^*$ to be the optimal policy  i.e. $\hat{\pi}_{s,a, \alpha, u}^* = \argmax \hat{V}_{s,a,\alpha, u}^{\pi}$. We will use the shorthand -- $\hat{V}_{\alpha, u}^{*} := \max \hat{V}_{s,a,\alpha, u}^{\pi}$ and $\hat{Q}_{\alpha, u}^{*} := \max \hat{Q}_{s,a,\alpha, u}^{\pi}$.

We consider a grid, \[U_{s,a} = \{0, \pm\iota(1-\gamma)/2, \pm2\iota(1-\gamma)/2, \pm3\iota(1-\gamma)/2 \ldots, \pm \alpha_{\max}\},\] and define 
$\Pi_{s,a} = \{\hat{\pi}_{s,a,\alpha, u}^*: u\in U_{s,a}\}$. 
Then 
$\Pi_{s,a}$ is a random set independent of $\hat{\cP}_{s,a}$.
Let $L = \{0, 1, \ldots, \lceil\log(e/(1-\gamma))\rceil\}$.
Then, by Lemma~\ref{lemma:bern}, we have, with probability at least $1-\delta/|S|/|A|$, 
for all $\pi \in \Pi_{s,a}$ and $l\in L$
\begin{align*}
\left|( \cP_{s,a}- \cPhat_{s,a}) \cdot V_{\beta}^{\pi, (l)} \right| & \leq 
\sqrt{\frac{2 \log(4|U_{s,a}||L||S||A|/\delta)}{N}} \sqrt{\Var{P_{s,a}}{V_{\beta}^{\pi, (l)}}} + \frac{2 \log(4|U_{s,a}||S||A||L|/\delta)\norminf{V_{\beta}^{\pi, (l)}}}{3 N},
\end{align*}
which we denote as event $\cE_{s,a}$.

Next, we show that if $u^* =\hat{Q}_{\alpha}^{*}(s,a) - \gamma \hat{V}_{\alpha}^{*}(s) $, then 
$\hat{\pi}_{s,a,\alpha, u^*}^* = \hat{\pi}_{\alpha}^*$: \\
(1) If $\pi^*(s) = a$, it straightforward to verify that \[
\hat{V}_{\alpha}^{*}(s) = \hat{Q}_{\alpha}^{*}(s,a)= u^* + \gamma \hat{V}_{\alpha}^{*}(s)
\ge r(s, a') + \hat{\cP}(\cdot|s', a')^\top 
\hat{V}_{\alpha}^{*}, \quad \forall a'\neq a.
\] \\
(2) If  $\pi^*(s) \neq a$, then 
\[
\hat{V}_{\alpha}^{*}(s) = \max_{a'}\hat{Q}_{\alpha}^{*}(s,a') = \max_{a'}(r(s, a') + \hat{\cP}(\cdot|s, a')^\top 
\hat{V}_{\alpha}^{*}) =  r(s, \hat{\pi}^*(s)) + \hat{\cP}(\cdot|s, \hat{\pi}^*(s))^\top 
\hat{V}_{\alpha}^{*}.\] 
(3) For $s'\neq s$, we have
\[\hat{V}_{\alpha}^{*}(s') = \hat{Q}_{\alpha}^*(s',\hat{\pi}^*(s')) = r(s', \hat{\pi}^*(s')) + \hat{\cP}(\cdot|s', \hat{\pi}^*(s'))^\top 
\hat{V}_{\alpha}^{*} = \max_{a'}(r(s, a') + \hat{\cP}(\cdot|s', a')^\top 
\hat{V}_{\alpha}^{*}).\] Therefore, $\hat{Q}_{\alpha}^*(s',a')$ and $\pihatopta$ satisfies the Bellman equations in the absorbing MDP; consequently, we have $\hat{Q}_{s,a,\alpha, u}^{\pihatopta}=\hat{Q}_{\alpha, u}^{*} =\hat{Q}_{\alpha}^*$ and $\hat{V}_{s,a,\alpha, u}^{\pihatopta} = \hat{V}_{\alpha, u}^{*} =\hat{V}_{\alpha}^{\pihatopta}$ and $\pihatopta$ is an optimal policy in the absorbing MDP.

Moreover, suppose event $\cE$ happens, then $\hat{Q}_{\alpha}^{*}$ satisfies the $\iota$-gap condition. By Lemma~\ref{policy-gap}, 
for all $|u-u^*|\le \iota(1-\gamma)/2$, we have
\[
\pihatopta = \hat{\pi}_{s,a,\alpha, u^*}^*
=\hat{\pi}_{s,a,\alpha, u}^*.
\]
Thus, if $\cE$ happens, then $\pihatopta = \hat{\pi}_{s,a,\alpha, u}^*$ for some $u\in U_{s,a}$ and thus $\hat{\pi}^*\in \Pi_{s,a}$.
On $\cE\cap \cE_{s,a}$, we have, for all $l\in L$,
\begin{align*}
\left|( \cP_{s,a}- \cPhat_{s,a}) \cdot V_{\beta}^{\pihatopta, (l)} \right| & \leq 
\sqrt{\frac{2 \log(4|U_{s,a}||L||S||A|/\delta)}{N}} \sqrt{\Var{P_{s,a}}{V_{\beta}^{\pihatopta, (l)}}} + \frac{2 \log(4|U_{s,a}||S||A||L|/\delta)\norminf{V_{\beta}^{\pihatopta, (l)}}}{3 N}.
\end{align*}

By a union bound over all $(s,a)$, we have, with probability at least $\Pr[\cE\cap (\cap_{s,a}\cE_{s,a})] \ge \Pr[\cE]-\delta$,  for all $(s,a)$, and $l\in L$,
\begin{align*}
\left|( \cP_{s,a}- \cPhat_{s,a}) \cdot V_{\beta}^{\pihatopta, (l)} \right| & \leq \sqrt{\frac{2 \log(4|U_{s,a}||L||S||A|/\delta)}{N}} \sqrt{\Var{P_{s,a}}{V_{\beta}^{\pihatopta, (l)}}} + \frac{2 \log(4|U_{s,a}||S||A||L|/\delta)\norminf{V_{\beta}^{\pihatopta, (l)}}}{3 N}.
\end{align*}

Let $\nu_1 = 2 \log(4 |U_{s,a}||S||A||L|/\delta)$ and apply Lemma~\ref{lemma:bern2con}, we arrive at,
\[
\norminf{\hat{V}_{\beta}^{\hat{\pi}^*} - V_{\beta}^{\hat{\pi}^*}} \le \frac{6}{1-\gamma}\cdot \sqrt{\frac{\nu_1}{N(1-\gamma)}} \norminf{\beta} 
\]
provided $N\ge \frac{16e^2}{1-\gamma}\nu_1$. For instantiating $\nu_1$, note that $|U_{s,a}| = \frac{4 \alpha_{\max}}{(1 - \gamma)^2 \iota}$, $|L| = \log\left(\nicefrac{e}{1 - \gamma}\right)$. Hence, $\nu_1 = 2 \log \left(\frac{16  \alpha_{\max} S A \log\left(\nicefrac{e}{1 - \gamma}\right)}{(1 - \gamma)^2 \, \iota \, \delta} \right)$. 

\end{proof}

\begin{thmbox}
\begin{lemma}
For any policy $\pi$, we have
\[
\norminf{\reward{\pi} - \rewardp{\pi}} \leq \frac{\omega}{1-\gamma} \quad \text{;} \quad \norminf{\rewardhat{\pi} - \rewardhatp{\pi}} \leq \frac{\omega}{1-\gamma}
\]
\label{lemma:perturbation-error}
\end{lemma}
\end{thmbox}
\begin{proof}
For policy $\pi$, $\reward{\pi} = \left(I - \gamma P_{\pi}\right)^{-1} r^{\pi}$ and $\rewardp{\pi} = \left(I - \gamma P_{\pi}\right)^{-1} r_p^{\pi}$. 
\begin{align*}
\reward{\pi} - \rewardp{\pi} &= \left(I - \gamma P_{\pi}\right)^{-1} [r^{\pi} - r_p^{\pi}] \\
\implies \norminf{\reward{\pi} - \rewardp{\pi}} & \leq \|\left(I - \gamma P_{\pi}\right)^{-1}\|_{1} \norminf{r^\pi - r^\pi_p} \\
\intertext{Since $\|\left(I - \gamma P_{\pi}\right)^{-1}\|_{1} \leq \frac{1}{1 - \gamma}$ and $\norminf{r^\pi - r^\pi_p} \leq \omega$}
\norminf{\reward{\pi} - \rewardp{\pi}} & \leq
\frac{\omega}{1 - \gamma}. 
\end{align*}
The same argument can be used to bound $\norminf{\rewardhat{\pi} - \rewardhatp{\pi}}$ completing the proof. 
\end{proof}

\mainconcentrationemp*
\begin{proof}
\begin{align*}
\abs{\rewardp{\pihatbar} - \rewardhatp{\pihatbar}} &= \abs{\frac{1}{T} \sum_{t = 0}^{T-1} \left[\rewardp{\pit} - \rewardhatp{\pit}\right]} \leq \frac{1}{T} \sum_{t = 0}^{T-1} \abs{\rewardp{\pit} - \rewardhatp{\pit}} \\
& \leq \frac{1}{T} \sum_{t = 0}^{T-1} \norminf{V_{r_p}^{\pit} - \hat{V}_{r_p}^{\pit}} 
\end{align*}
Recall that $\hat{M}_{r + \lambdat c}$ satisfies the gap condition with $\iota = \frac{\omega \, \delta}{30 \, |\Lambda| |S||A|^2}$ for every $\lambdat \in \Lambda$. Since $|\Lambda| = \frac{U}{\epsl}$, $\iota = \frac{\omega \, \delta (1 - \gamma) \, \epsl}{30 \, U |S||A|^2}$. Since $\pit := \argmax \hat{V}_{r_p + \lambdat c}^{\pi}$, we use~\cref{lemma:main} with $\alpha = r_p + \lambdat c$ and $\beta = r_p$, and obtain the following result. For $N \geq \frac{4 \, C(\delta)}{1-\gamma}$, for each $t \in [T]$, with probability at least $1 - \delta/5$,
\begin{align*}
\norminf{V_{r_p}^{\pit} - \hat{V}_{r_p}^{\pit}} & \leq \sqrt{\frac{C(\delta)}{N \cdot (1-\gamma)^3 }} \, (1 + \omega) \leq 2 \sqrt{\frac{C(\delta)}{N \cdot (1-\gamma)^3 }} 
\end{align*}
Using the above relations, 
\begin{align*}
\abs{\rewardp{\pihatbar} - \rewardhatp{\pihatbar}} & \leq 2 \sqrt{\frac{C(\delta)}{N \cdot (1-\gamma)^3 }} 
\end{align*}
Similarly, invoking~\cref{lemma:main} with $\alpha = r_p + \lambdat c$ and $\beta = c$ gives the bound on $\abs{\const{\pihatbar} - \consthat{\pihatbar}}$. 
\end{proof}

\begin{thmbox}
\mainconcentrationopt*
\end{thmbox}
\begin{proof}
Since $\piopt$ and $\pi^*_c$ are fixed policies independent of the sampling, we can directly use~\citet[Lemma 1]{li2020breaking}. 
\end{proof}

\clearpage
\subsection{Helper lemmas}
\label{app:common-helper}

\begin{thmbox}
\begin{lemma}
With probability at least $1 - \delta/10$, for every $\lambda \in \Lambda$, $\hat{M}_{r_p + \lambda c}$ satisfies the gap condition in~\cref{def:gap-condition} with $\iota = \frac{\omega \, \delta \, (1-\gamma)}{30 \, |\Lambda||S||A|^2}$. 
\label{lemma:gap}
\end{lemma}
\end{thmbox}
\begin{proof}
Using~\cref{lemma:gapli} with a union bound over $\Lambda$ gives the desired result. 
\end{proof}

\begin{thmbox}
\begin{lemma}
\label{policy-gap}
Let $\pi^{*}_{\alpha}$ and $\pi^{*}_{\alpha'} $ be two optimal policies for an MDP with rewards $\alpha$ and $\alpha'$ respectively.
Suppose $Q_\alpha^*$ satisfies the $\iota$-gap condition. 
Then, for all $\alpha'$ with $\norminf{\alpha' - \alpha} < \iota(1-\gamma)/2$, we have
\[
\pi^{*}_{\alpha} = \pi^{*}_{\alpha'}.
\]
\label{lemma:abs-lipschitz-policy}
\end{lemma}
\end{thmbox}
\begin{proof}
Since $\norminf{\alpha' - \alpha} < \iota(1-\gamma)/2$, we thus have
\[
\|{Q}^*_{\alpha} - {Q}^*_{\alpha'}\|_{\infty}
< \frac{\iota(1-\gamma)}{2(1-\gamma)} = \frac{\iota}{2}.
\]
Note that, for all $s$,
\begin{align*}
{Q}^*_{\alpha}(s,{\pi}^*_{\alpha'}(s))
&>{Q}^*_{\alpha'}(s,{\pi}^*_{\alpha'}(s)) - \frac{\iota}{2}\\
&\ge {Q}^*_{\alpha'}(s, {\pi}^*_{\alpha}(s)) - \frac{\iota}{2},\\
&> {Q}^*_{\alpha}(s, {\pi}^*_{\alpha}(s)) - {\iota}\\
& \ge  {Q}^*_{\alpha}(s, a'). \quad\forall a'\neq {\pi}^*_{\alpha}(s).
\end{align*}
Hence, ${\pi}^*_{\alpha'}(s)\not\in\{a': a'\neq {\pi}^*_{\alpha}(s)\}$ for all $s\in S$, 
and consequently ${\pi}^*_{\alpha} = {\pi}^*_{\alpha'}$.
\end{proof}

\begin{thmbox}
\begin{lemma}[Lemma~6 of~\citet{li2020breaking}]
\label{lemma:gapli}
Consider the MDP $M=(S,A, P, \gamma, r_p)$ with randomly perturbed rewards $r_p$ ($r_p(s,a) = r(s,a) + \xi(s,a)$ where $\xi(s,a) \sim \mathcal{U}[0,\omega]$ and $r(s,a)\in \mathbb{R}$). If optimal $Q$-function is denoted as $Q^*_{r_p}$ 
and  $\pi_{r_p}^*$ is an optimal deterministic  policy for $M$
 then
 for any $\delta\in [0,1]$,
  with probability at least $1-\delta$ we have, for all $s$ and $a\neq \pi_{r'}^*(s)$
\[
Q^{*}_{r_p}(s,\pi_{r_p}^*(s)) \ge 
Q^{*}_{r_p}(s, a) + \frac{\omega \delta (1-\gamma)}{3|S||A|^2} \, 
\]
\end{lemma}
i.e. $M$ satisfies the gap condition in~\cref{def:gap-condition} with $\iota = \frac{\omega \delta (1-\gamma)}{3|S||A|^2}$. 
\end{thmbox}

\begin{thmbox}
\begin{lemma}[Bernstein inequality] 
\label{lemma:bern}
Fix a state $s$, an action $a$. Let  $\delta \geq 0$. Then, for any fixed vector $V$, with probability greater than $1-\delta$, it holds that, 
\begin{align*}
\left|( \cP_{s,a}- \cPhat_{s,a}) \cdot V \right| & \leq
\sqrt{\frac{2 \log(4/\delta)}{N}} \sqrt{\Var{P_{s,a}}{V}} + \frac{2 \log(4/\delta)\norminf{V}}{ 3 N}.
\end{align*}
\label{lemma:bernstein}
\end{lemma}
\end{thmbox}
\clearpage
\section{Lower Bound}
\label{app:lb}
In this section, we first present the lower-bound in the simplified bandit setting (\cref{app:lb-bandit}), and then present the formal CMDP lower bound in~\cref{app:proof-lb}. 

\subsection{Bandit Setting}
\label{app:lb-bandit}
Consider the 2-arm bandit setting where arm-1 has a mean reward $H := \frac{1}{1 - \gamma}$ and mean constraint reward $b-\zeta$ and arm-2 has a mean reward $0$ and mean constraint reward $b+\zeta$. For the constraint RHS equal to $b$ (implying that the Slater constant is $\zeta$), the ground truth optimal policy plays each arm with equal probability to achieve a reward value $H/2$ (with no constraint violation). \\ \\
However, suppose there is an error $\epsilon'$ in the estimation of the constraint reward in arm-2 (we estimate it to be $b + \zeta - \epsilon'$); then even if everything else is exactly estimated, the empirical optimal policy has to play action $2$ with probability $\zeta/(2\zeta-\epsilon')\approx 1/2 + \epsilon'/(4\zeta)$ to satisfy the constraint, which gives a value $H/2 - H\epsilon'/(4\zeta)$. This results in an $H\epsilon'/(4\zeta)$  error of the final policy. To obtain an $\epsilon$-correct policy without constraint violation, one needs to set $\epsilon' = \epsilon\zeta/H = \epsilon \zeta \, (1 - \gamma)$, thus inflating the sample complexity compared to the unconstrained setting. We instantiate this intuition for CMDPs in~\cref{thm:lb-strict}.

\subsection{Proof of~\cref{thm:lb-strict}}
\label{app:proof-lb}

\lbstrict* 

\begin{figure}[!ht]
    \subfigure
    {
    \includegraphics[width=0.9\textwidth]{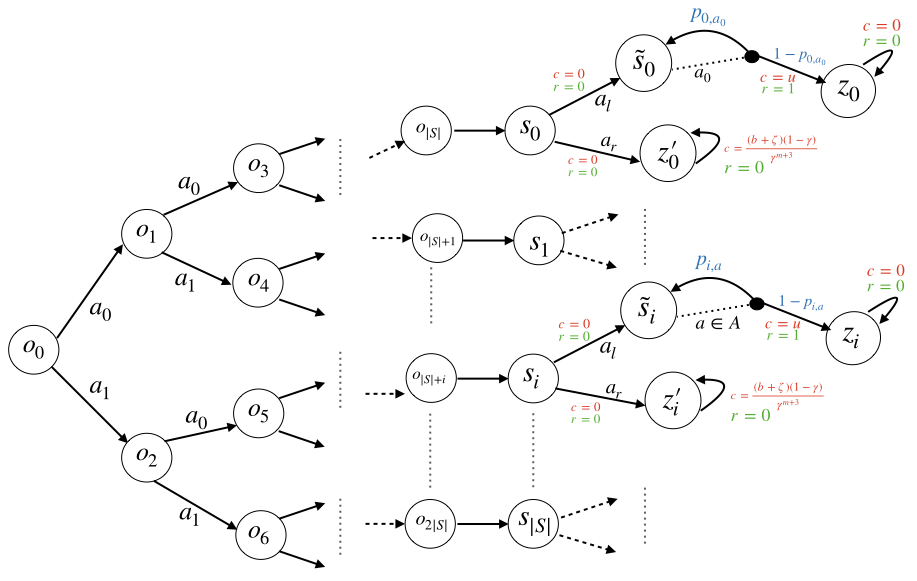}
    }
    \caption{
    The lower-bound instance consists of CMDPs with $S = 2^m - 1$ (for some integer $m > 0$) states and $A$ actions. We consider $SA + 1$ CMDPs -- $M_0$ and $M_{i,a}$ ($i \in \{1,\ldots S\}$, $a \in \{1,\ldots, A\}$) that share the same structure shown in the figure. For each CMDP, $o_0$ is the fixed starting state and there is a deterministic path of length $m+1$ from $o_0$ to each of the $S + 1$ states -- $s_i$ (for $i \in \{0, 1, \ldots, S\}$). Except for states $\tilde{s}_i$, the transitions in all other states are deterministic. For $i \neq 0$, for action $a \in \cA$ in state $\tilde{s}_i$, the probability of staying in $\tilde{s}_i$ is $p_{i,a}$, while that of transitioning to state $z_i$ is $1 - p_{i,a}$. There is only one action $a_0$ in $\tilde{s}_0$ and the probability of staying in $\tilde{s}_0$ is $p_{0,a_0}$, while that of transitioning to state $z_0$ is $1 - p_{0,a_0}$. The CMDPs $M_0$ and $M_{i,a}$ only differ in the values of $p_{i,a}$. The rewards $r$ and constraint rewards $c$ are the same in all CMDPs and are denoted in green and red respectively for each state and action.  
    } 
    \label{fig:lb}
\end{figure}

\begin{proof}[Proof of Theorem~\ref{thm:lb-strict}]
Without loss of generality, we assume $|\cS|=2^m - 1$ for some integer $m$.
In what follows, we first introduce our hard instance. Note that some of the states of this instance may have less than $|\cA|$ actions. This is without loss of generality as one can easily duplicate actions to make each state having exactly $|\cA|$ actions. 

\textbf{Hard Instance.}
We will consider the basic gadget defined in Figure~\ref{fig:lb}. 
We will make $|\cS|+1$ copies of this gadget. 
In the $i$-th gadet for $i=0,1,\ldots, |\cS|$, 
there is an input state $s_i$ with 2 actions $a_l, a_r$. Playing $a_l$ deterministically transitions to $\tilde{s}_i$ with reward and constraint $0$. Playing $a_r$ goes to $z'_i$ with reward and constraint $0$. 
In state $\tilde{s}_i$, there are $|\cA|$ actions (only 1 action $a_0$ on $\tilde{s}_0$). Playing action $a\in \cA$, this state transitions to state $z_i$ with probability $1-p_{i,a}$ and self loop with probability $p_{i,a}$, where $p_{i,a}$ will be determined shortly. The reward at this state is $1$ and constraint reward is $u$ to be determined. In state $z_i'$, there is only one action, whose reward is $0$ and constraint is $(b+\zeta)(1-\gamma)/\gamma^{m+3}$.

Lastly, we consider $2|\cS|+1$ routing state $o_0, o_1, \ldots$ that form a binary tree, i.e., in each of these routing state, there are two actions $a_0$ and $a_1$. The state-action pair $(o_i, a_j)$ transitions to state $o_{2i + j+1}$ for $i< |\cS|$.
Each state $o_{|\cS|+k}$ transitions deterministically to the gate state $s_k$ for $k=0,1,\ldots |\cS|$. Rewards and constraint are all 0 for these states.
Note that, for any state $s_k$, there is a unique deterministic path of length $m+1$ connecting $o_0$. 
Hence we require $\gamma \ge  1 - 1/(cm)$ for some constant $c$ and hence
$\gamma^{\Theta(m)} = \Theta(1)$.

Note that this instance modifies the MDP instance in \cite{feng2019does}. Some parameters chosen are also adapted from there.

\textbf{Hypotheses.}
With the above defined hard instance structure, we now define a family of hypotheses on the probability transitions. Later we will show that any sound algorithm would be able to test the hypothesises but would require a large number of samples.
Let $q_0, q_1, q_2 \in (0, 1)$ be some parameters to be determined. We define,
\begin{itemize}
    \item \emph{Null case} MDP $M_0$: $p_{0,a_0} = q_1$, and $p_{i, a} = q_0$ for all $i\in [|\cS|]$ and $a\in \cA$. 
    \item \emph{Alternative case} MDP $M_{i, a}$: $p_{0,a_0} = q_1$,  $p_{j, a'} = q_0$ for all $(j, a')\neq (i,a)$, and $p_{i,a} = q_2$. 
\end{itemize}
Note that all these MDPs $M_0, M_1\ldots $ share exactly  the same graph structure, with only the transition probability changes. Moreover, $M_{i,a}$ differs from $M_0$ only on state-action pair $(\tilde{s}_i, a)$.

\textbf{Optimal Policies.}
Now we specify $q_0, q_1$ and $q_2$ and check the optimal policies of each hypothesis.
Let $\epsilon' = (1-\gamma)\zeta\epsilon$,
\[
{q}_0 = \frac{1-{c_1(1-\gamma)}}{\gamma},
q_1 = {q}_0+\alpha_1,  \text{ and }q_2 = {q}_0+\alpha_2,
\]
and 
\[
\alpha_1 = \frac{c_2(1-\gamma {q}_0)^2\epsilon'}{\gamma}, \quad\text{and}
\quad \alpha_2 = \frac{c_3(1-\gamma {q}_0)^2\epsilon'}{\gamma}
\]
for some absolute constants $c_1, c_2, c_3 > 0$
and that $\alpha_1/q_0\in (0, 1/2)$, $\alpha_1/(1-q_0)\in (0, 1/2)$, $\alpha_2/q_0\in (0, 1/2)$, and $\alpha_2/(1-q_0)\in (0, 1/2)$.
We choose these parameters such that 
\[
\frac{1}{c_1(1-\gamma)} = \frac{1}{1-\gamma q_0}< \frac{1}{1-\gamma q_1}< \frac{1}{1-\gamma q_2} = \frac{1}{c_1(1-\gamma)} + c_4\epsilon'
\]
for some constants $c_4>0$
and that 
\[
\left|\frac{1}{1-\gamma q_1} - \frac{1}{1-\gamma q_0} \right| = \Theta(\epsilon'), \text{ and }
\left|\frac{1}{1-\gamma q_1} - \frac{1}{1-\gamma q_2} \right| = \Theta(\epsilon').
\]
Note that, for the reward values, if $\zeta=\Theta(1)$, then these actions only differ by $\Theta(\epsilon')\ll \epsilon$. A correct algorithm would not need to distinguish these actions. Yet, once the constraints are concerns, we will show shortly that these actions do differ because of the constraint values.

With these parameters, we can derive the optimal CMDP policy for $M_{0}$ and each $M_{i, a}$.
For any policy $\pi$, we denote the 
state occupancy measure as $\mu^{\pi}$, i.e., 
\begin{align}
\label{eqn:flux}
\mu^{\pi}(s,a) = \sum_{z}\rho(z)\sum_{t=1}^{\infty}\gamma^{t-1}\sum_{z}
\Pr_{\pi}[s_t=s, a_t=a|s_1=z]
\end{align}
where $\rho$ is the initial distribution with $\rho(o_0) = 1$ and $\rho(s)=0$ for all $s\neq o_0$.
This occupancy measure describes the discounted reachablity from $o_0$ to an arbitrary state action pair.
The reward value and constraint value can be written as 
\[
V^{\pi} = \sum_{s,a} \mu^{\pi}(s,a)r(s,a), \quad \text{and}\quad V_{c}^{\pi} = 
\sum_{s,a}\mu^{\pi}(s,a) c(s,a).
\]
Note that, given a state-occupancy measure $\mu$, a policy can be specified as $\pi_{\mu}(a|s) = \mu(s,a)/\sum_{a'} \mu(s, a')$. We can use the LP formulation for the CMDP as follows
\begin{align}
    \max &\sum_{s,a} \mu(s,a)r(s,a) \text{ subject to },\notag\\
        &\forall s: ~\sum_{a}\mu(s,a) = \rho(s) + \gamma \sum_{s', a'} P(s|s', a')\mu(s', a'),\notag\\
        &\sum_{s,a}\mu^{\pi}(s,a) c(s,a) \ge b,\notag\\
        &\mu(s,a)\ge 0.
\end{align}
Due to the structure of the CMDPs, we can further simply the structure of the constraints. In particular, we have
\begin{align}
\forall i:\quad 
\sum_{a}\mu(\tilde{s}_i, a) = \gamma\mu({s}_i, a_l) + \gamma \sum_{a}p_{i, a}\mu(\tilde{s}_i, a) 
\end{align}
let  $\bar{\nu} = \sum_{i} \mu(s_i, a_r)$. We then have
\[
\sum_{i}\mu({s}_i, a_l) + \bar{\nu} = \gamma^{m+2}, 
\]
and
\[
\mu(z_i', a) = \frac{\gamma\mu(s_i, a_r)}{1-\gamma}.
\]
Consider $M_0$,  let $\mu_0 := \mu({s}_0,a_l)$,  $\mu^c_0:= \sum_{i>0} \mu({s}_i,a_l)$ we have
\[
\sum_{i>0, a} \mu(\tilde{s}_i,a) = \frac{\gamma \mu^c_0}{1-\gamma q_0},
\quad \text{and}\quad \mu(\tilde{s}_0,a_0)
= \frac{\gamma \mu_0}{1-\gamma q_1}.
\]
The LP can be rewritten as 
\begin{align*}
    &\max \quad \frac{\gamma\mu_0}{1-\gamma q_1} + \frac{\gamma\mu_0^c}{1-\gamma q_0} \\ & \text{ s.t. } 
    \frac{\gamma\mu_0 \cdot u}{1-\gamma q_1} + \frac{\gamma\mu_0^c \cdot u}{1-\gamma q_0} + \frac{ \bar{\nu}\cdot (b+\zeta)}{\gamma^{m+2}} \ge b, ~ \mu_0 +\mu_0^c +\bar{\nu} = \gamma^{m+2}, \text{  and  } \mu(s,a)\ge 0, ~\forall (s,a). 
\end{align*}

Here, we specify the values of $u$ as,
\[
u = \frac{(1-\gamma q_0) (b-x)}{\gamma^{m+3}},
\]
for some $x=\Theta(\zeta)$,
such that,
\[
 \frac{\gamma^{m+3}u}{1-\gamma q_0} = b-x< \frac{\gamma^{m+3}u}{1-\gamma q_1}
= b-x + \epsilon'_1< \frac{\gamma^{m+3}u}{1-\gamma q_2} = b-x+\epsilon'_2< b,
\]
where $c'\epsilon'\le \epsilon'_1 \le \epsilon'_1  + c'' \epsilon' \le \epsilon'_2 \le c'''\epsilon'$ for some positive constants $c', c'', c'''$ determined by $c_1, c_2, c_3$. For this value of $u$, the maximum value of the constraint value function $\max V_c^\pi$ is $b + \zeta$, implying that $\zeta$ is the Slater constant for all these CMDPs.

Thus, for $M_0$, the solution is, 
\[
\mu_0 = \frac{\gamma^{m+2}\zeta}{\zeta + x - \epsilon'_1},\quad
\mu_0^c=  0,   \quad\text{and}\quad \bar{\nu} = \frac{(x-\epsilon'_1)\gamma^{m+2}}{\zeta+x - \epsilon'_1}.
\]

Note that this implies the policy deterministically choose a path to reach state $s_0$, and then plays action $a_l$ with probability $\mu_0/\gamma^{m+2}$.
The optimal value in this case is \[
V_{M_0}^*(o_0) = \frac{\zeta}{\zeta + x -\epsilon'_1}\cdot \frac{\gamma^{m+3}}{1-\gamma q_1}. 
\]

Similarly, for $M_{i, a}$ with $i\ge 1$, 
let $\mu_{i,a} := \mu({s}_i, a)$, $\mu_{i,a}^c := \sum_{i'>0, (i', a')\neq (i,a)}\mu(s_{i'}, a')$, the LP can be written as 
\begin{align*}
    &\max \quad \frac{\gamma\mu_0}{1-\gamma q_1} + \frac{\gamma\mu_{i,a}^c}{1-\gamma q_0} 
    + \frac{\gamma\mu_{i,a}}{1-\gamma q_2} \\ & \text{ s.t. } 
   \frac{\gamma\mu_0u}{1-\gamma q_1} + \frac{\gamma\mu_{i,a}^cu}{1-\gamma q_0} 
    + \frac{\gamma\mu_{i,a}u}{1-\gamma q_2} + \frac{ \bar{\nu}\cdot (b+\zeta)}{\gamma^{m+2}} \ge b, ~ \mu_0 +\mu_{i,a}^c + \mu_{i,a} +\bar{\nu} = \gamma^{m+2}, \\
    &\text{  and  } \mu(s,a)\ge 0, ~\forall (s,a). 
\end{align*}

the solution is
\[
\mu_{i,a}=\frac{\zeta\gamma^{m+2}}{\zeta+x-\epsilon'_2},~
\mu_{i,a}^c = \mu_0 = 0, \quad\text{and}\quad \bar{\nu} =\frac{(x-\epsilon'_2)\gamma^{m+2}}{\zeta+x - \epsilon'_2},
\]
i.e., the policy chooses a path to reach state $s_i$ deterministically and with optimal value
\[
V_{M_{i,a}}^*(o_0) = \frac{\zeta}{\zeta+x-\epsilon'_2}\cdot\frac{\gamma^{m+3}}{1-\gamma q_2}. 
\]

Lastly, we shall check the gap of the value functions.
\[
\left|V_{M_{i,a}}^*(o_0) - V_{M_{0}}^*(o_0)\right|
\ge \left(\frac{\zeta}{\zeta+x-\epsilon'_2} -\frac{\zeta}{\zeta+x-\epsilon'_1}\right)\cdot \frac{\gamma^{m+3}}{1-\gamma q_1}
= \frac{\epsilon''}{\zeta + x} \cdot  \frac{\gamma^{m+3}}{1-\gamma q_1}
\ge c_7 \epsilon.
\]
where $\epsilon''\ge c_8 \epsilon'$ for some constants $c_7, c_8$ determined by $c_1, c_2, c_3$ and $x$.
Thus the error in $V_c^{\pi}$ is amplified by a factor of $\Theta[(1-\gamma)^{-1}\zeta^{-1}]$.

\textbf{Implications of Soundness: Near-Optimal Policies.}
Let $\cK$ be a $(\epsilon, \delta)$-sound algorithm, i.e., on input any CMDP with a generative model, it outputs a policy, which is $\epsilon$-optimal with probability at least $1-\delta$. 
We thus define the event
\[
\cE_{0} = \left\{\cK \text{ outputs policy $\pi$ such that } \mu^{\pi}({s}_0, a_l) \ge \frac{\zeta\gamma^{m+2}}{(\zeta + x - \epsilon_1'/2)}\right\},
\]
i.e., this event requires the output policy reaching $s_0$ and play action $a_l$ with sufficiently high probability.
We now measure the probability of $\cE_{0}$ on different input CMDPs. 
Due to the soundness, it must be the case that
\[
\Pr_{M_{i,a}}[\cE_{0}] < \delta.
\]
If not, on $\cE_{0}$, 
$\mu^{\pi}({s}_0, a_l)\ge \frac{\zeta \gamma^{m+2}}{\zeta + x - \epsilon_1'/2}$,
we can then compute the best possible $V^{\pi}_{M_{i,a}}(o_0)$ as solving the following LP,
\begin{align*}
    &\max \quad \frac{\gamma\mu_0}{1-\gamma q_1} + \frac{\gamma\mu_{i,a}^c}{1-\gamma q_0} 
    + \frac{\gamma\mu_{i,a}}{1-\gamma q_2} \\ & \text{ s.t. } 
   \frac{\gamma\mu_0u}{1-\gamma q_1} + \frac{\gamma\mu_{i,a}^cu}{1-\gamma q_0} 
    + \frac{\gamma\mu_{i,a}u}{1-\gamma q_2} + \frac{ \bar{\nu}\cdot (b+\zeta)}{\gamma^{m+2}}\ge b, ~ \mu_0 +\mu_{i,a}^c + \mu_{i,a} +\bar{\nu} = \gamma^{m+2}, \\
    & \mu_0 \ge  \frac{\zeta\gamma^{m+2}}{(\zeta + x - \epsilon_1'/2)} \text{  and  } \mu(s,a)\ge 0, ~\forall (s,a). 
\end{align*}

Plugging the values of $p_{i,a}$, and due to $q_0 < q_1 < q_2$, we obtain  $\mu_0 = \frac{\zeta\gamma^{m+2}}{\zeta+x-\epsilon_1'/2}$, $\mu_{i,a}^c = 0$ and,
\begin{align*}
    \mu_{i,a}\cdot (b-x + \epsilon'_2)
     + \mu_0\cdot (b-x + \epsilon'_1)
     + \bar{\nu}(b+\zeta) = b\gamma^{m+2}, \mu_{i,a}+ \mu_0 + \bar{\nu} = \gamma^{m+2}.
\end{align*}
Hence, 
\[
\mu_{i,a}
= \gamma^{m+2}\cdot\frac{\zeta - \mu_0 (\zeta+x-\epsilon'_1)}{\zeta+x-\epsilon_2'}
\le 
\frac{c_9\epsilon_1'\zeta\cdot \gamma^{m+2}}{2(\zeta+x-\epsilon_1'/2)(\zeta+x-\epsilon_2)}
\]
for some constant $c_9$ depending on $c_1, c_2, c_3, x$, 
and 
\begin{align*}
V^{\pi}(o_0) & \le 
\frac{\zeta}{\zeta+x-\epsilon_1'/2}\cdot\frac{\gamma^{m+3}}{1-\gamma q_1} +
\frac{c_9\epsilon_1'\zeta}{2(\zeta+x-\epsilon_1'/2)(\zeta+x-\epsilon_2)}\cdot\frac{\gamma^{m+3}}{1-\gamma q_2}
\end{align*}
Note that
\[
V^*_{M_{i,a}}(o_0) = \frac{\zeta}{\zeta + x-\epsilon_2'}\cdot \frac{\gamma^{m+3}}{1-\gamma q_2}
\]
and \[
V^{*}_{M_{i,a}}(o_0) - 
V^{\pi}_{M_{i,a}}(o_0)
\ge \left(\frac{\zeta}{\zeta+x-\epsilon'_2} -\frac{\zeta}{\zeta+x-\epsilon'_1/2}-\frac{c_9\epsilon_1'\zeta}{2(\zeta+x-\epsilon_1'/2)(\zeta+x-\epsilon_2')}\right)\cdot \frac{\gamma^{m+3}}{1-\gamma q_2}
\ge  \epsilon
\]
for some appropriately chosen $c_1, c_2, c_3, x$, which is
a contradiction of the $(\epsilon, \delta)$-soundness.

\textbf{Implications of Soundness: Expectation on Null Hypothesis.}
Let $N_{i,a}$ be the number of samples the algorithm $\cK$ takes on state-action $(s_i, a)$. Next we show that, 
$\mathbb{E}[N_{i,a}]$ has to be big on $M_0$.
\begin{thmbox}
\begin{lemma}
Let $t_* = \frac{c_{10}\log\delta^{-1}}{(1-\gamma)^3{\epsilon'}^2}$ for some constant $c_{10}$.
For any $(\epsilon, \delta)$-sound algorithm $\cK$, for any $(i,a)$, we have \[
\mathbb{E}_{M_0}(N_{i,a})
\ge t_*.
\]
\end{lemma}
\end{thmbox}
\begin{proof}
Suppose $\mathbb{E}_{M_0}(N_{i,a})<t_*$, then we aim to show a contradiction: $\Pr_{M_{i,a}}[\cE_0]\ge \delta$.
Similar to the proof above, 
since $\cK$ is $(\epsilon, \delta)$-sound, it must be the case that 
\[
\Pr_{M_{0}}[\cE_{0}]\ge 1-\delta.
\]
We now consider the likelihood ratio
\[
\Pr_{M_{i,a}}[\cE_{0}] / \Pr_{M_{0}}[\cE_{0}].
\]
For any realization of the empirical samples, consider the samples the algorithm takes as $\tau = \{(s_{i,a}^{(1)}, s_{i,a}^{(2)}, \ldots, s_{i,a}^{(N_{i,a})}): (i,a)\in [|\cS|]\times [\cA]\}$. 
Let us define $N_{s', s, a}$ as the number of samples from $(s, a)\to s'$.
By Markov property, since the only difference of the probability matrix between $M_{i,a}$ and $M_{0}$ is on $p_{i, a}$, we have
\begin{align*}
\frac{\Pr_{M_{i,a}}[\tau]}{\Pr_{M_{0}}[\tau]}
&= \frac{q_2^{N_{\tilde{s}_i, \tilde{s}_i,a}}(1-q_2)^{N_{z_i, \tilde{s}_i,a}}}{q_0^{N_{\tilde{s}_i, \tilde{s}_i,a}}(1-q_0)^{N_{z_i, \tilde{s}_i,a}}}= \left(\frac{q_2}{q_0}\right)^{N_{\tilde{s}_i, \tilde{s}_i,a}}\cdot\left(\frac{1-q_2}{1-q_0}\right)^{N_{z_i, \tilde{s}_i,a}}\\
&= \left(1+\frac{\alpha_2}{q_0}\right)^{Nq_0-\Delta}\cdot \left(1-\frac{\alpha_2}{1-q_0}\right)^{N(1-q_0)+\Delta}
\end{align*}
where $\Pr_{M}[\tau]$ denotes the probability of $\cK$ taking the samples $\tau$ in CMDP $M$, 
 $\Delta = N_{i,a}{q}_0 - N_{\tilde{s}_i, \tilde{s}_i, a}$, and 
 $N=N_{i,a}$.

By a similar derivation of Lemma~5 in \cite{feng2019does} (page 15-19), 
on the following event, 
\[
\cE'_{i,a} = \left\{N_{i,a} \le 10t_*, \text{ and } |N_{\tilde{s}_i, \tilde{s}_i,a}
- N_{i,a}q_0| \le \sqrt{20(1-q_0)q_0N_{i,a}}\right\}
\]
we have
\[
\frac{\Pr_{M_{i,a}}[\tau]}{\Pr_{M_{0}}[\tau]}
\ge 4\delta
\]
provided appropriately chosen $c_1, c_2, c_3, x, c_{10}$.

By Markov inequality and Doob's inequality (e.g. Lemma~7-8 of \cite{feng2019does}), we have
\[
\Pr_{M_0}[\cE'_{i,a}] \ge  1 - \frac{1}{10} - \frac{1}{10} = \frac{4}{5}.
\]

We are able to compute the  probability of $\cE_{0}$ on $M_{i,a}$ as follows:
\begin{align*}
\Pr_{M_{i,a}}[\cE_{0}]
 &= \sum_{\tau \in \cE_{0}} \Pr_{M_{i,a}}[\tau] 
 \ge  
 \sum_{\tau \in \cE_{0}\cap \cE_{i,a}'} \Pr_{M_{i,a}}[\tau] \\
 &= \sum_{\tau \in \cE_{0}\cap \cE_{i,a}'} \frac{\Pr_{M_{i,a}}[\tau]}{\Pr_{M_{0}}[\tau]}\cdot\Pr_{M_{0}}[\tau]
 \ge 4\delta \sum_{\tau \in \cE_{0}\cap \cE_{i,a}'} \Pr_{M_{0}}[\tau]
 \ge 3\delta,
\end{align*}
provided $\delta \le c_{11}$ for some absolute constant $c_{11}$, hence a contradiction of soundness.
\end{proof}
\paragraph{Wrapping up.}Hence, if the algorithm is $(\epsilon, \delta)$-sound for all $\{M_{i,a}\}$, it must be the case that 
\[
\mathbb{E}_{M_0}[N_{i,a}] \ge t_*, \forall (i,a)\in [|\cS|]\times \cA.
\]
By linearity of expectation, we have
\[
\mathbb{E}_{M_0}\left[\sum_{i,a}N_{i,a}\right] \ge |\cS||\cA|t_*.
\]
Since $\epsilon' = \epsilon (1 - \gamma) \zeta$, $\mathbb{E}_{M_0}\left[\sum_{i,a}N_{i,a}\right] \geq \frac{c_{10} |\cS| |\cA| \, \log\delta^{-1}}{(1-\gamma)^5 \zeta^2 {\epsilon}^2}$, which completes the proof.
\end{proof}

\section{Estimating $\zeta$}
\label{app:est-zeta}
In this section, we show that $\zeta$ can be estimated up to error $0.2\zeta$ using small number of samples.
The formal guarantee is provided in the following theorem.
\begin{theorem}
There exists an algorithm that, with probability at least $1-\delta$, halts, takes 
\[
{O}\left(\frac{c_{max}{\log\left(\frac{|\cS||\cA|}{(1-\gamma)\delta\zeta}\right)}}{(1-\gamma)^3\zeta^2}\right)
\]
samples per state-action pair, and outputs an estimator $\hat{\zeta}$, such that
\[
|\hat{\zeta} - \zeta| \le 0.2\zeta.
\]
\end{theorem}
\begin{proof}
Let $\zeta_i = 2^{-i}/(1-\gamma)$ for $i=0, 1,2, \ldots, $
and 
\[
N_i = \frac{c_{\max}C_i(\delta)}{(1-\gamma)^3\zeta_i^2},
\]
where
\[
C_i(\delta) = {c'\log\left(2\frac{|\cS||\cA|i^2}{(1-\gamma)\zeta_i\delta}\right)}
\]
for some constant $c'$.
We start by
running the algorithm in \cite{li2020breaking} for $N_i$ samples per state-action pair on the MDP with $c$  as reward , for $i=0, 1, \ldots$ and stop if the following is satisfied:
\begin{align*}
|\hat{V}_{c,i}^*(\rho) - b|  \ge 9\zeta_i,
\end{align*}
where $\hat{V}_{c,i}^*$ is the empirical optimal value function obtained for using $N_i$ samples.
Then we output $\hat{\zeta} = \hat{V}_{c,i}^*(\rho)$.

Next we show that the algorithm halts.
Let $\cE_i$ be event for iteration $i$, 
\[
|\hat{V}_{c,i}^*(\rho) - {V}_{c}^*(\rho)|\le \zeta_i.
\]
Thus, by Theorem~1 of \cite{li2020breaking},
\[
\Pr[\cE_i] \ge 1-\frac{\delta}{2i^2}.
\]
Next, let  $i^*$ be such that  $0.05\zeta\le \zeta_{i^*} < 0.1\zeta$. 
Hence, if  $\cE_{i^*}$ happens, then 
\[
|\hat{V}_{c,i^*}^*(\rho) - {V}_{c}^*(\rho)|\le \zeta_{i^*}
\]
and
\[
|{V}_{c}^*(\rho) - b| - |\hat{V}_{c,i^*}^*(\rho) - {V}_{c}^*(\rho)| \le |\hat{V}_{c,i^*}^*(\rho) - b|. 
\]
Hence, on $\cE_{i^*}$, 
\[
|\hat{V}_{c,i^*}^*(\rho) - b|\ge \zeta - \zeta_{i^*}\ge 0.9\zeta \ge 9\zeta_{i^*}
\]
and the algorithm halts at least before iteration $i^*$.

Next, suppose the algorithm halts at $i\le i^*$, then on $\cE_{i}$, we have
\[
|\hat{V}_{c,i}^*(\rho) - {V}_{c}^*(\rho)|\le \zeta_{i},
\]
\[
|\hat{V}_{c,i}^*(\rho) - b|
\ge 9 \zeta_{i} \ge 9|\hat{V}_{c,i}^*(\rho) - {V}_{c}^*(\rho)|,
\]
and
\[
|(|\hat{V}_{c,i}^*(\rho) - b| - 
|{V}_{c,i}^*(\rho) - b|)| 
\le |\hat{V}_{c,i}^*(\rho) - {V}_{c}^*(\rho)|\le |\hat{V}_{c,i}^*(\rho) - b|/9.
\]
Note that $\zeta = |{V}_{c}^*(\rho) - b|$, we have
\[
|\hat{\zeta} - \zeta| \le  \hat{\zeta}/9\implies 
\zeta \ge 8/9\hat{\zeta} ~ \text{and}~ |\hat{\zeta} - \zeta|
\le {\zeta}/8.
\]
Thus, on event 
$\cE = \cE_1\cap \cE_2\cap \cE_3 \ldots \cE_{i^*}$, which happens with probability at least 
\[
1-\sum_{i=1}^{i^*}\frac{\delta}{2i^2} \ge 1-\frac{\pi^2\delta}{12},
\]
 we have $|\hat{\zeta} - \zeta| \le \zeta / 8$, proving the correctness.
 
We now consider the overall sample complexity. Suppose $\cE$ happens, then the number of samples consumed is upper bounded by
\[
\sum_{i=1}^{i^*} N_i \le 
\frac{c''c_{max}C_{i^*}(\delta)}{(1-\gamma)^3\zeta^2}
=\frac{c''c'c_{max}{\log\left(\frac{|\cS||\cA|}{(1-\gamma)\delta\zeta}\right)}}{(1-\gamma)^3\zeta^2},
\]
for some constant $c''$, completing the proof.
\end{proof}

\section{Comparison to~\citet{bai2021achieving}}
\label{app:bai-comparison}
For a target error of $\epsilon$, our lower-bound construction in the strict feasibility setting (\cref{sec:lb-strict}) shows that it is important to estimate the constraint value function to a smaller error equal to $\epsilon'$. Intuitively, this is because a small estimation error in the constraint value can (incorrectly) render the optimal policy infeasible and result in a $\frac{\text{Range(value function)}}{\zeta}$ suboptimality gap. \\ \\
For the theoretical results in~\citep{bai2021achieving}, the value functions are normalized and hence $\text{Range(value function)} = 1$. Hence, a small constraint violation can result in an $\frac{1}{\zeta}$ suboptimality gap, and the constraint value function needs to be estimated to a smaller error equal to $\epsilon' := \epsilon \zeta$. Combining this with the standard results for the primal-dual algorithm for unconstrained MDPs~\citep{wang2020randomized} for \emph{normalized value functions}, this implies a sample-complexity of $O\left(\frac{S A}{(1 - \gamma)^2 \, \epsilon'^2}\right) = O\left(\frac{S A}{(1 - \gamma)^2 \, \epsilon^2 \zeta^2}\right)$ which is the sample complexity reported in~\citep[Theorem 2]{bai2021achieving}. \\ \\
On the other hand, if we scale the value function to lie to the standard $O\left(\nicefrac{1}{1 - \gamma}\right)$ range, a small constraint violation can result in an $\frac{1}{(1 - \gamma) \zeta}$ suboptimality gap, and the constraint value function needs to be estimated to a smaller error equal to $\epsilon' := \epsilon (1 - \gamma) \zeta$. The rescaling also affects the sample complexity for the primal-dual algorithm for unconstrained MDPs~\citep{wang2020randomized}. Specifically, for unconstrained MDPs, if the value functions lie in the $\left[0, \nicefrac{1}{1 - \gamma}\right]$ range, the primal-dual algorithm in~\citet{wang2020randomized} requires $O\left(\frac{S A}{(1 - \gamma)^4 \, \epsilon^2}\right)$ samples. Since we require a smaller error $\epsilon'$ in the strict feasibility setting, this implies an $O\left(\frac{S A}{(1 - \gamma)^4 \, \epsilon'^2}\right) = O\left(\frac{S A}{(1 - \gamma)^{6} \, \epsilon^2 \zeta^2}\right)$ sample complexity.

\end{document}